\documentclass[letterpaper]{article} 
\usepackage{aaai23}  
\usepackage{times}  
\usepackage{helvet}  
\usepackage{courier}  
\usepackage[hyphens]{url}  
\usepackage{graphicx} 
\usepackage{natbib}  
\usepackage{caption} 
\usepackage{algorithm, algpseudocode}
\usepackage{amsmath, amssymb, mathtools, amsthm}
\usepackage{tabulary, multirow}
\usepackage{wrapfig, booktabs, makecell}
\usepackage{enumitem, subcaption}
\usepackage{hyperref}

\theoremstyle{plain}
\newtheorem{theorem}{Theorem}[section]

\newtheorem{lemma}[theorem]{Lemma}

\newtheorem{remark}[theorem]{Remark}
\theoremstyle{definition}
\newtheorem{definition}[theorem]{Definition}
\newtheorem{assumption}[theorem]{Assumption}

\usepackage{newfloat}
\usepackage{listings}
\DeclareCaptionStyle{ruled}{labelfont=normalfont,labelsep=colon,strut=off} 
\floatstyle{ruled}
\newfloat{listing}{tb}{lst}{}
\floatname{listing}{Listing}
\usepackage{kotex, xcolor}

\title{A Gift from Label Smoothing: Robust Training with \\ Adaptive Label Smoothing via Auxiliary Classifier under Label Noise}
\author {
    Jongwoo Ko\textsuperscript{\rm 1}\thanks{The two authors contributed equally.},
    Bongsoo Yi\textsuperscript{\rm 2}$^{*}$,
    Se-Young Yun\textsuperscript{\rm 1}
}
\affiliations {
    \textsuperscript{\rm 1} Kim Jaechul Graduate School of AI, KAIST, Seoul, Korea \\
    \textsuperscript{\rm 2} Department of Statistics and Operations Research, University of North Carolina at Chapel Hill \\
    \{jongwoo.ko, yunseyoung\}@kaist.ac.kr, bongsoo@unc.edu
}

\usepackage{bibentry}

\begin{document}

\maketitle

\begin{abstract}
As deep neural networks can easily overfit noisy labels, robust training in the presence of noisy labels is becoming an important challenge in modern deep learning.
While existing methods address this problem in various directions, they still produce unpredictable sub-optimal results since they rely on the posterior information estimated by the feature extractor corrupted by noisy labels.
Lipschitz regularization successfully alleviates this problem by training a robust feature extractor, but it requires longer training time and expensive computations.
Motivated by this, we propose a simple yet effective method, called ALASCA, which efficiently provides a robust feature extractor under label noise.
ALASCA integrates two key ingredients: (1) adaptive label smoothing based on our theoretical analysis that label smoothing implicitly induces Lipschitz regularization, and (2) auxiliary classifiers that enable practical application of intermediate Lipschitz regularization with negligible computations.
We conduct wide-ranging experiments for ALASCA and combine our proposed method with previous noise-robust methods on several synthetic and real-world datasets.
Experimental results show that our framework consistently improves the robustness of feature extractors and the performance of existing baselines with efficiency.
Our code is available at \href{https://github.com/jongwooko/ALASCA}{https://github.com/jongwooko/ALASCA}.
\end{abstract}
\section{Introduction}

While deep neural networks (DNNs) have high expressive power that leads to promising performances, the success of DNNs heavily relies on the quality of training data, in particular, accurately labeled training examples. 
Unfortunately, labeling large-scale datasets is a costly and error-prone process, and even high-quality datasets contain incorrect labels\,\cite{nettleton2010study, zhang2017understanding}.
Hence, mitigating the negative impact of noisy labels is critical, and many approaches have been proposed to improve robustness against noisy data for learning with noisy labels (LNL).

Robustness to label noise is typically pursued by identifying noisy samples to reduce their contribution to the loss\,\cite{han2018co, mirzasoleiman2020coresets}, correcting labels\,\cite{yi2019probabilistic, li2020dividemix}, utilizing a robust loss function\,\cite{zhang2018generalized,wang2019symmetric}.
However, one of the biggest challenges of LNL methods involves providing a dependable criterion for distinguishing clean data from noisy data, such that clean data is fully exploited while filtering noisy data.
While these existing methods are partially effective in mitigating label noise, their criterion for identifying noisy examples uses biased posterior information from a linear classifier or the penultimate layer of the corrupted network. These unpredictable biases can lead to a reduction in the network's ability to separate clean and noisy instances \cite{nguyen2020self, kim2021fine}.


To solve this undesired bias, several regularization methods \cite{xia2020robust, cao2021heteroskedastic} have been proposed to enhance the robustness of the feature extractor. However, while existing regularization-based learning frameworks alleviate the degradation, these methods require multiple training stages and considerable computational costs and are difficult to apply in practice. \citet{cao2021heteroskedastic} used two-stage training to compute the relative data-dependent regularization power to conduct Lipschitz regularization (LR) on intermediate layers. \citet{xia2020robust} identified and regularized the non-critical parameters that tend to fit noisy labels and require longer training time. Some studies\,\cite{zhang2020decoupling, zheltonozhskii2022contrast} have designed contrastive learning frameworks to generate high-quality feature extractors using unsupervised approaches, which require considerable computations for high performance.

To mitigate these impractical issues, we provide a simple yet effective learning framework for a robust feature extractor, \textbf{A}daptive \textbf{LA}bel \textbf{S}moothing via auxiliary \textbf{C}l\textbf{A}ssifier (ALASCA), with theoretical guarantee and small additional computation. Our proposed method is robust to label noise itself and can further enhance the performance of existing LNL methods. Our main contributions are as follows: 
\begin{itemize}
    \item We theoretically explain that label smoothing (LS) implicitly induces LR, which is known to enable robust training with noisy labels \cite{finlay2018lipschitz, cao2021heteroskedastic}. 
    Through theoretical motivations, we empirically show that adaptive LS (ALS) can regularize noisy examples while fully exploiting clean examples.
    \item To practically implement adaptive LR on the intermediate layers, we propose ALASCA, which combines ALS with auxiliary classifiers. To the best of our knowledge, this is the first study to apply auxiliary classifiers under label noise with theoretical evidence.
    \item We experimentally demonstrate that ALASCA is universal by combining various LNL methods and validating that ALASCA consistently boosts robustness on benchmark-simulated and real-world datasets. 
    \item We verify that ALASCA effectively enhances the robustness of feature extractors by comparing the quality of subsets on sample-selection methods and robustness to the hyperparameter selection of LNL methods.
\end{itemize}
\section{Related Works}
\subsection{Learning with Noisy Labels}
\citet{zhang2017understanding} empirically demonstrated that convolutional neural networks trained with stochastic gradient methods easily memorize random labeling of the training data. To address this, numerous studies have examined the classification task with noisy labels. Existing methods address this problem by (1) filtering noisy examples and training using only clean examples\,\cite{han2018co, mirzasoleiman2020coresets, kim2021fine} or (2) relabeling noisy examples using the model itself or another model trained only on the clean dataset\,\cite{lee2018cleannet, li2020dividemix}. Some approaches focus on designing loss functions with robust behaviors and provable tolerance to label noise\,\cite{ghosh2017robust,zhang2018generalized,wang2019symmetric}. We fully describe these previous works in Appendix~\ref{app:related}.

\paragraph{Regularization-based Methods.} Another line of work has attempted to design regularization-based techniques. For example, some studies have stated and theoretically analyzed how early-stopped model can prevent the memorization phenomenon of noisy labels\,\cite{arpit2017closer, song2019does}. Based on this, \citet{liu2020early} proposed an early learning regularization\,(ELR) loss function that avoids memorizing noisy data by leveraging semi-supervised learning\,(SSL) techniques. \citet{xia2020robust} clarified that neural network parameters cause memorization and proposed a robust training method for these parameters. Developing regularization at the prediction level has been addressed by smoothing one-hot vectors\,\cite{lukasik2020does} and distilling the rescaled predictions of other models \cite{muller2019does,kim2021comparing}. Recently, \citet{cao2021heteroskedastic} proposed a heteroskedastic adaptive regularization that applies stronger regularization to noisy instances.

\subsection{Label Smoothing}
LS\,\cite{szegedy2016rethinking, muller2019does} is commonly used to construct a generalized DNN model by preventing over-confident predictions. This regularization technique facilitates generalization by softening a ground-truth one-hot vector $\mathbf{y}$ with a weighted mixture of hard targets:
\begin{equation*}
    \mathbf{y}^{LS}:=(1-\alpha) \cdot \mathbf{y} +  \frac{\alpha}{L}  \cdot \mathbf{1}_L,
\end{equation*}
where $L$ denotes the number of classes, $\mathbf{1}_L$ denotes an all-one vector in $\mathbb{R}^{L}$, and $\alpha \in [0,1]$ is a smoothing parameter. 
\citet{lukasik2020does} claimed that LS denoises label noise by causing label correction and weight-shrinkage regularization effects. 
However, \citet{wei2021understanding} recently show that LS tends to over-smooth the estimated posterior under high levels of label noise, which can hurt robustness.
Moreover, several studies \cite{szegedy2016rethinking, pereyra2017regularizing, muller2019does, chorowski2017towardsBD} have validated that LS boosts model generalizability, and \citet{li2020regularization} proposed the need for data-dependent smoothing.
\citet{li2020regularization} proposed structural LS, which selects smoothing strength data-dependently that minimizes the Bayes error rate bias.
\citet{ghoshal2021learningBS} derived PAC Bayesian generalization bounds for LS and proposed adaptive smoothing for the latent structure of the label space. 

\section{Methodology: ALASCA}
LR has been shown to be effective for DNNs\,\cite{gouk2021regularisation}. \citet{wei2019data, wei2019improved} theoretically and empirically show that LR
for all intermediate layers improves generalization of DNNs. 
In Figure~\ref{fig:lr}, we observe that LR prevents overfitting noisy data and enhances generalization under label noise, which supports previous studies.
Furthermore, many studies show that different regularization strengths along the data points are essential. \citet{wang2013smoothing} and \citet{tibshirani2014adaptive} stated that smoothing splines with different smoothing parameters perform well in regression problems. Recently, \citet{cao2021heteroskedastic} showed that applying strong LR to highly uncertain data points improves generalization.
We recall key takeaways from \citet{cao2021heteroskedastic} that motivated our work.
\begin{remark}[\citealt{cao2021heteroskedastic}]\label{method:cao}
\textit{In a binary classification problem on one-dimensional data, the authors
     \lowercase\expandafter{\romannumeral1}) derived the formula for the asymptotic mean squared error (MSE) on the test set,
     \lowercase\expandafter{\romannumeral2}) with some simplifications, showed that the asymptotic MSE is minimized when the smoothing parameter is proportional to the $\frac{3}{5}$-th power of the label uncertainty.}
\end{remark}
The exact theorem statement and detailed explanation may be found in Appendix~\ref{appendix:thm1}.
However, this explicit regularization requires multiple training phases to estimate and apply relative regularization power for different data points.
Furthermore, computing the Hessian matrix resulting from directly regularizing the norm of Jacobian matrices increases the computational cost \cite{filiposka2014complexity, nesser2021reduced}.
In this section, we present that LS implicitly incurs LR and introduce the simple unified framework for efficient learning with label noise, ALASCA. In principle, our method can be used with most LNL methods, such as noise-robust loss function \cite{zhang2018generalized, wang2019symmetric} and sample-selection methods \cite{han2018co, kim2021fine}.

\subsection{Label Smoothing as Lipschitz Regularization}
In this section, we analytically present our motivation that LS implicitly encourages LR.
Here, we formally define the notation and terminology for our problem. 

\begin{figure*}[t]
    \hspace*{\fill}
    \begin{subfigure}[b]{0.245\linewidth} 
    \centering
    \includegraphics[width=\linewidth]{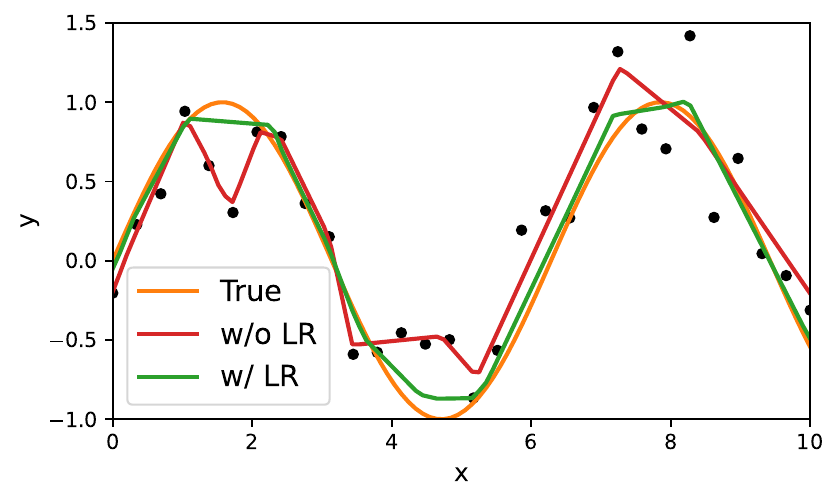}
    \caption{Effectiveness of LR}
    \label{fig:lr}
    \end{subfigure}
    \hfill
    \begin{subfigure}[b]{0.23\linewidth} 
    \centering
    \includegraphics[width=\linewidth]{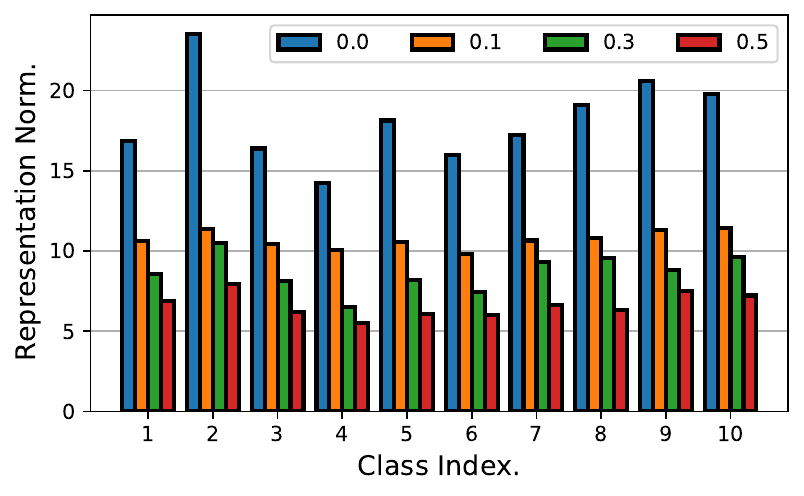}
    \caption{Representation Norm}
    \end{subfigure}
    \hfill
    \begin{subfigure}[b]{0.245\linewidth} 
    \centering
    \includegraphics[width=\linewidth]{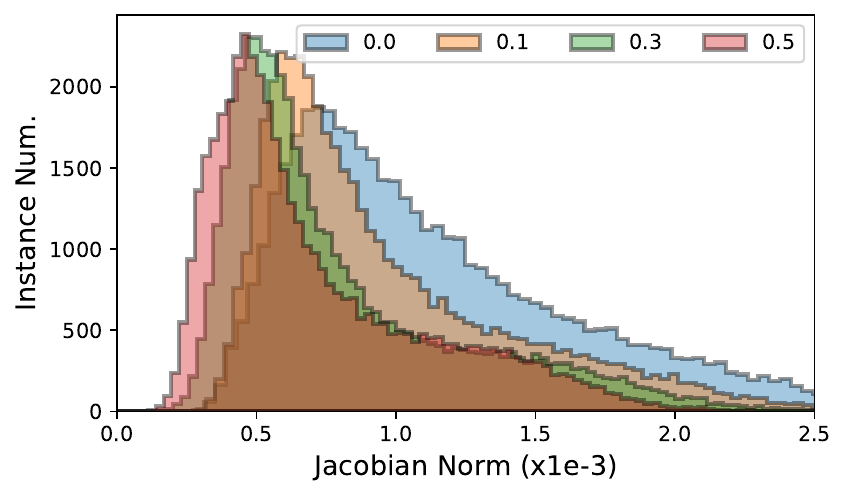}
    \caption{Jacobian Norm}
    \end{subfigure}
    \hspace*{\fill}
    \caption{(a) One-dimensional experiments under label noise with and without LR. We observe that LR effectively enhance the generalization under label noise. Comparison of (b) class-wise representation vector norm (Theorem~\ref{method:theorem1}), and (c) distribution of Jacobian matrix norm for penultimate layer (Theorem~\ref{method:theorem2}) across different smoothing factors on CIFAR-10.}\label{fig:ls}
\end{figure*}

\paragraph{Notation.} We focus on multiclass classification with $L$ classes. Assume that the data points and labels lie in $\mathcal{X} \times \mathcal{Y}$, where the feature space $\mathcal{X} \subset \mathbb{R}^{D}$ and label space $\mathcal{Y} = \left\{ 0, 1 \right\}^{L}$. A single data point $\mathbf{x}$ and its label $\mathbf{y}$ follow a distribution $(\mathbf{x}, \mathbf{y}) \sim P_{\mathcal{X} \times \mathcal{Y}}$. We aim to find a predictor $\mathbf{f}: \mathcal{X} \rightarrow \mathbb{R}^{L}$ minimizing the risk of $\mathbb{E}_{(\mathbf{x}, \mathbf{y}) \sim P_{\mathcal{X} \times \mathcal{Y}}} \left[ \ell(\mathbf{f}(\mathbf{x}), \mathbf{y}) \right]$ with loss function $\ell : \mathbb{R}^{L} \times \mathbb{R}^{L} \rightarrow \mathbb{R}_{+}$.

\begin{definition}[Lipschitzness] A function $f$ is called Lipschitz continuous with a Lipschitz constant $L_{f} \in [0, \infty)$ if
\begin{equation*}
    \| f(y) - f(x) \| \leq L_{f} \| y - x \|, 
\end{equation*}
for all $x, y \in \text{\textup{dom }} f$.
\end{definition}

\begin{definition}[Lipschitz Regularization]
Let $\mathcal{F}$ be a twice-differentiable model family from $\mathbb{R}^{D}$ to $\mathbb{R}^{L}$. Lipschitz regularization aims to optimize the function with a smoothness penalty as follows:
\begin{equation*}
    \min_{\mathbf{f} \in \mathcal{F}} \frac{1}{N} \sum_{n=1}^{N} \ell (\mathbf{f}(\mathbf{x}_{n}), \mathbf{y}_{n}) + \lambda \| \mathbf{J}_{\mathbf{f}} (\mathbf{x}_{n}) \|_{F},
\end{equation*}
where $\lambda$ is the regularization coefficient, $N$ the number of training data points, $\mathbf{J}_{\mathbf{f}}$ the Jacobian matrix of $\mathbf{f}$, and $\| \cdot \|_{F}$ the Frobenius norm.
\end{definition}
Here, we focus only on $\ell$ as cross-entropy (CE) loss. Compared with the CE loss with one-hot vector $\ell(\mathbf{f}(\mathbf{x}), \mathbf{y})$, the CE loss with LS $\ell(\mathbf{f}(\mathbf{x}), \mathbf{y}^{LS})$ of factor $\alpha$ can be presented as follows:
\begin{align*}
        \ell(\mathbf{f}(\mathbf{x}), \mathbf{y}^{LS}) &= (1 - \alpha) \cdot \ell(\mathbf{f}(\mathbf{x}), \mathbf{y}) + \frac{\alpha}{L} \cdot \ell(\mathbf{f}(\mathbf{x}), \mathbf{1}_{L}) \\
    & \propto  \ell(\mathbf{f}(\mathbf{x}), \mathbf{y}) + \frac{\alpha}{(1-\alpha) \cdot L} \cdot \Omega(\mathbf{f}).
\end{align*}
By denoting $f_{i}(\cdot)$ as the $i$-th element of logit vector  $\mathbf{f}(\cdot)$, the regularization term of LS is
\begin{equation}\label{eq:ls_regularization}
    \Omega(\mathbf{f}) = L \cdot \log \left[ \sum_{i=1}^{L} e^{f_{i}(\cdot)} \right] - \sum_{i=1}^{L} f_{i}(\cdot).
\end{equation}
Previously, \citet{lukasik2020does} suggested that LS encourages weight shrinkage in DNNs; however, this interpretation was validated only for linear models. 
To better understand LS, we consider that the DNN function can be presented as arbitrary surrogate models. 
We suppose that the surrogate model $\mathbf{f}(\cdot) = \mathbf{g} \circ \mathbf{h}(\cdot)$ consists of an arbitrary twice-differentiable feature extractor $\mathbf{h}: \mathcal{X} \rightarrow \mathbb{R}^{Q}$ and fixed linear classifier $\mathbf{g}(\mathbf{z}) = \mathbf{W}^{\intercal} \mathbf{z}$ with $\mathbf{W} \in \mathbb{R}^{Q \times L}$ and $\mathbf{z} \in \mathbb{R}^{Q}$. 
First, we find a minimizer of the regularization term of LS. The following assumption is required to guarantee the uniqueness of the minimizer. 

\begin{assumption}\label{thm:assumption1}
    $\mathbf{W}_1, \mathbf{W}_2, \cdots, \mathbf{W}_L$ is an affine basis of $\mathbb{R}^{Q}$, where $\mathbf{W}_i$ is the $i$-th column of $\mathbf{W}$ (i.e., $\mathbf{W}_2-\mathbf{W}_1, \mathbf{W}_3-\mathbf{W}_1, \cdots, \mathbf{W}_L-\mathbf{W}_1$ are linearly independent).
\end{assumption}

\begin{theorem}\label{method:theorem1}
$\mathbf{h}=\mathbf{0}$ is a minimizer of $\Omega \circ \mathbf{g}$. \, If Assumption \ref{thm:assumption1} holds, $\mathbf{h}=\mathbf{0}$ is the unique minimizer.
\end{theorem}

Note that $\mathbf{h}=\mathbf{0}$ is always a minimizer of Equation \eqref{eq:ls_regularization} without any assumptions.
The takeaway from Theorem \ref{method:theorem1} is that the regularization term of LS encourages $\mathbf{h}(\mathbf{x}_1),\dots,\mathbf{h}(\mathbf{x}_N)$ to shrink to zero.
However, this does not ensure the shrinkage of the Jacobian matrix of $\mathbf{f}$.
The following theorem shows that this is true under a common Lipschitz assumption. We state the assumption and present our next main result.

\begin{assumption}\label{thm:assumption2}
    Each gradient $\nabla h_i(\mathbf{x})$ is Lipschitz continuous with a Lipschitz constant $L_h$ for all $i$, where $\mathbf{h}(\mathbf{x})=(h_1(\mathbf{x}), h_2(\mathbf{x}), \cdots, h_Q(\mathbf{x}))$.
\end{assumption}

\begin{theorem}\label{method:theorem2}
Consider a bounded feature space $\mathcal{X}$ and suppose that Assumption \ref{thm:assumption2} is satisfied. If $\mathbf{h}(\mathbf{x}_n)=\mathbf{0}$ for $1\leq n \leq N$, $\left \|  \mathbf{J_f}(\mathbf{x}_n)\right \|_F \rightarrow \mathbf{0}$ as $N \rightarrow \infty$ holds for $1\leq n \leq N$.
\end{theorem}
Theorem~\ref{method:theorem2} states that the smoothness of the feature extractor function in the training set induces LR. 
By combining Theorem~\ref{method:theorem1} and \ref{method:theorem2}, we conclude that the regularization term of LS encourages LR. 
The detailed proof of the theorems may be found in Appendix~\ref{appendix:thm2}. To validate our theoretical perspective, we conducted
the following exploratory experiments: compare the (1) class-wise average norm values of representation vectors and (2) Jacobian matrix norm from the penultimate layer of ResNet34 trained on CIFAR-10. 
As Figure~\ref{fig:ls} shows, both the representation and Jacobian matrix norms decrease to zero as the smoothing factor increases, indicating that the LS regularization term induces the representation vector to the origin (Theorem~\ref{method:theorem1}) and implicitly incurs LR (Theorem~\ref{method:theorem2}).

\begin{figure*}[t]
    \hspace*{\fill}
    \begin{subfigure}[b]{0.2835\textwidth} 
    \centering
    \includegraphics[width=\linewidth]{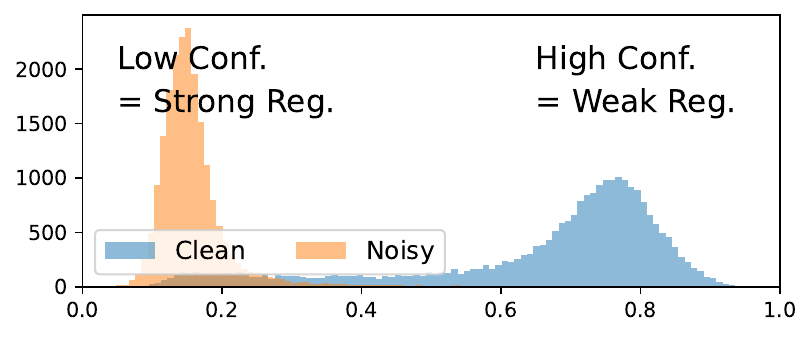}
    \caption{Concept of ALS}
    \end{subfigure}
    \hfill
    \begin{subfigure}[b]{0.27\textwidth} 
    \centering
    \includegraphics[width=\linewidth]{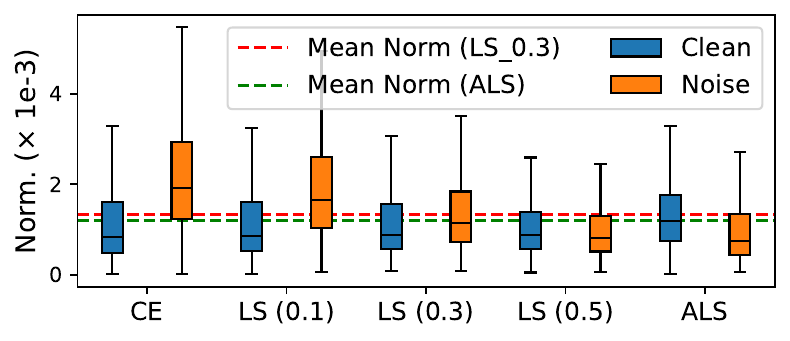}
    \caption{Jacobian Norm (Symm 50\%)}
    \end{subfigure}
    \hfill
    \begin{subfigure}[b]{0.27\textwidth} 
    \centering
    \includegraphics[width=\linewidth]{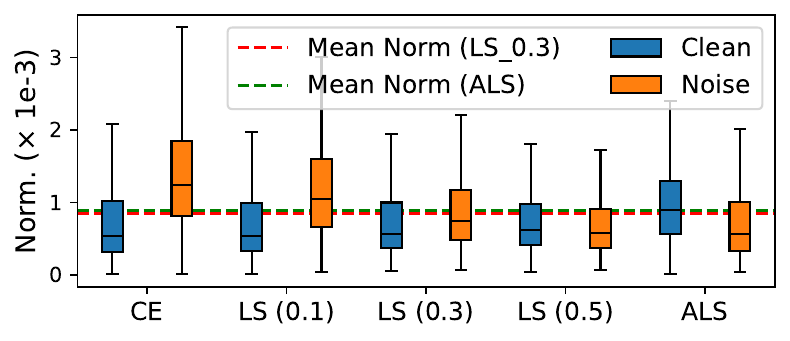}
    \caption{Jacobian Norm (Asym 40\%)}
    \end{subfigure}
    \hspace*{\fill}
    \caption{(a) The concept of ALS. Because noisy instances tend to have a lower confidence, we conduct stronger regularization on lower confidence instances. (b), (c) Comparison of Jacobian matrix norms for the penultimate representation on CIFAR-10 with 50\% of symmetric and 40\% of asymmetric noise, respectively. While the mean of the Jacobian norm across all instances are similar in LS (0.3) and ALS, ALS applies different powers of regularization for clean and noisy instances.}\label{fig:als}
\end{figure*}

\paragraph{Adaptive Regularization.} 
From our perspective of LS, we can apply adaptive LR using different smoothing factors.
\citet{cao2021heteroskedastic} shows that applying stronger LR to highly uncertain data points improves generalization on noisy datasets. To implement adaptive LR through LS, we design the smoothing factor of LS proportional to the $1 - \Pr (y|\mathbf{x})$ for each instance. 
While \citet{cao2021heteroskedastic} suggests that the optimal smoothing parameter is proportional to the 3/5-th power of the label uncertainty, our proposed strategy shows similar performances despite its simplicity.
Consequently, we use the following loss function for ALS:
\begin{equation}\label{eq:als}
    \ell^{ALS}_{\tilde{\alpha}(\mathbf{x})}(\mathbf{f}(\mathbf{x}), \mathbf{y}) = (1 - \tilde{\alpha}(\mathbf{x})) \cdot \ell (\mathbf{f}(\mathbf{x}), \mathbf{y}) + \tilde{\alpha}(\mathbf{x}) \cdot \Omega(\mathbf{f}),
\end{equation}
where $\tilde{\alpha}(\mathbf{x})$ is the $y$-th element of $1-\mathcal{S}(\mathbf{f}(\mathbf{x}))$, and $\mathcal{S}$ is the softmax function. 
Figure~\ref{fig:als} shows that the proposed ALS enables us to mainly regularize noisy examples because we conduct stronger regularization for lower confidence examples, where noisy examples take up a large proportion. While uniform LS (ULS) equally regularizes the smoothness of both clean and noisy examples, we validate that ALS can apply the appropriate LR for each example.

\begin{figure}[t]
    \hspace*{\fill}
    \begin{subfigure}[b]{0.48\linewidth} 
    \centering
    \includegraphics[width=\linewidth]{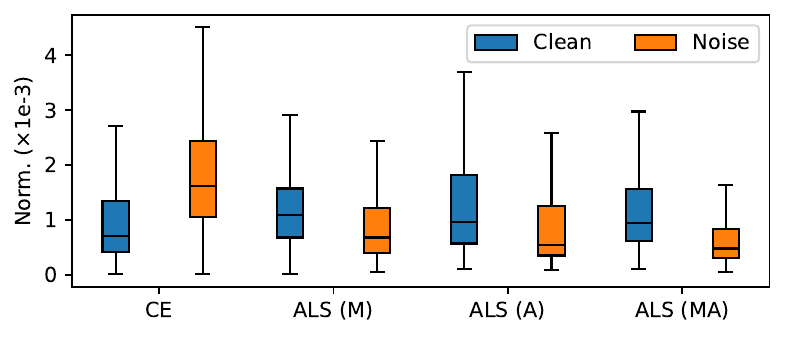}
    \caption{Jacobian Norm}
    \end{subfigure}
    \hfill
    \begin{subfigure}[b]{0.465\linewidth} 
    \centering
    \includegraphics[width=\linewidth]{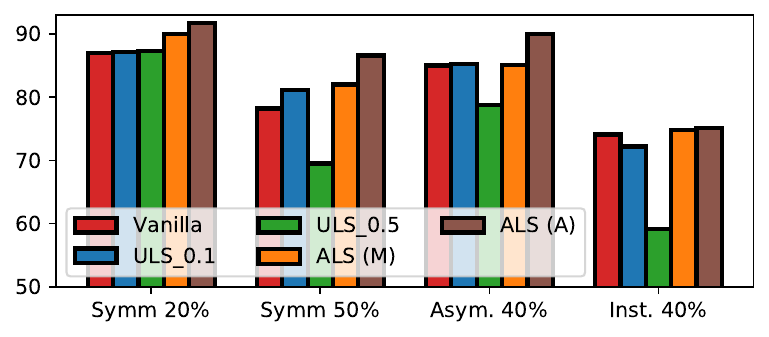}
    \caption{Performance}
    \end{subfigure}
    \hspace*{\fill}
    \caption{Comparisons of (a) Jacobian matrix norm of penultimate layer on CIFAR-10 under 50\% of symmetric noise; (b) performance on CIFAR-10 under various label noise for conducting uniform LS (ULS) and ALS on main classifier\,(M), auxiliary classifier\,(A), and both (MA).}\label{fig:auxiliary}
\end{figure}

\begin{figure}[t]
    \hspace*{\fill}
    \begin{subfigure}[b]{0.43\linewidth} 
    \centering
    \includegraphics[width=\linewidth]{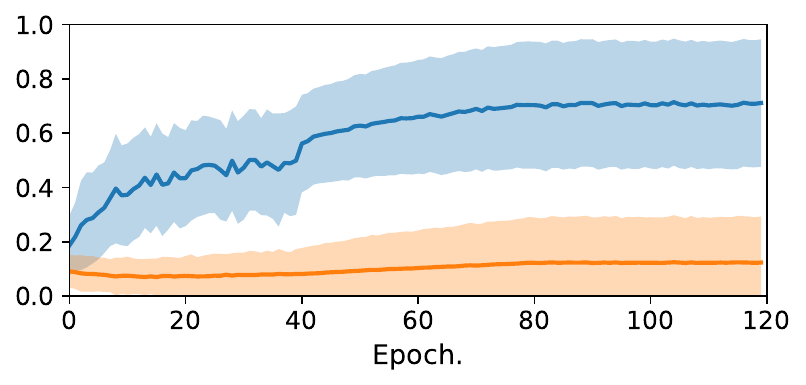}
    \caption{Symm 50\% (Inst.)}
    \end{subfigure}
    \hfill
    \begin{subfigure}[b]{0.43\linewidth} 
    \centering
    \includegraphics[width=\linewidth]{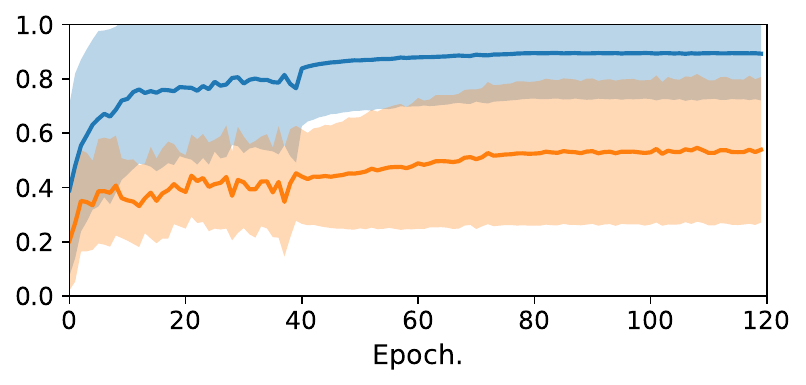}
    \caption{Asym 40\% (Inst.)}
    \end{subfigure}
    \hspace*{\fill} \\
    \hspace*{\fill}
    \begin{subfigure}[b]{0.43\linewidth} 
    \centering
    \includegraphics[width=\linewidth]{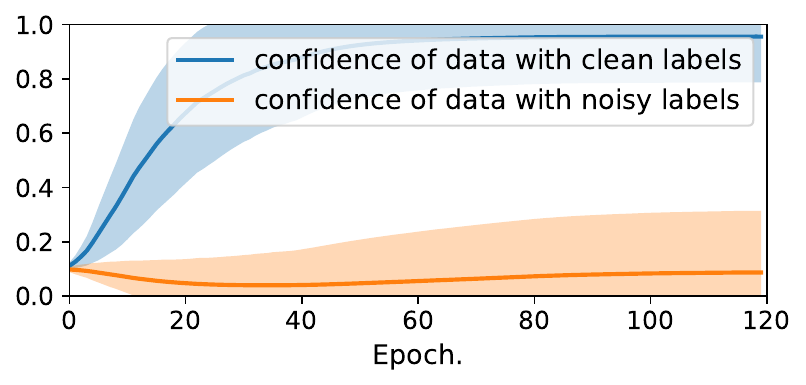}
    \caption{Symm 50\% (EMA)}
    \end{subfigure}
    \hfill
    \begin{subfigure}[b]{0.43\linewidth} 
    \centering
    \includegraphics[width=\linewidth]{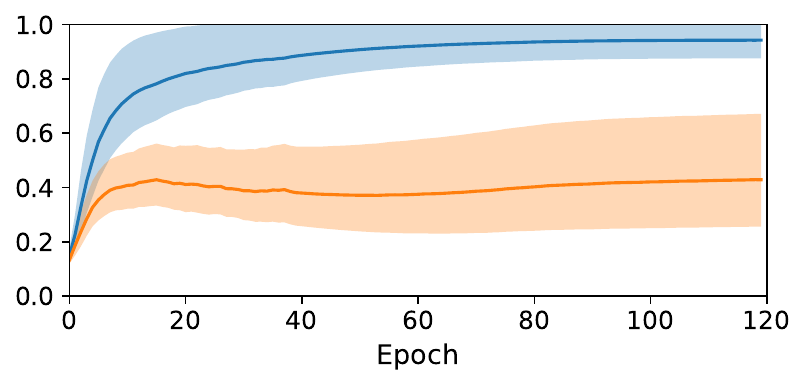}
    \caption{Asym 40\% (EMA)}
    \end{subfigure}
    \hspace*{\fill}
    \caption{Dynamic patterns (mean $\pm$ std) of instantaneous (first row) and EMA confidence (second row) suggested in ALASCA for CIFAR-10 under various noise types.}\label{fig:ema}
\end{figure}

\subsection{Combination with Additional Techniques}\label{sec:3.2}
Now, we integrate ALS with auxiliary classifiers (AC) and the exponential moving averaged (EMA) confidence to practically regularize the smoothness of intermediate layers. We describe the overall algorithm of ALASCA in Algorithm~\ref{alg:alasca}.

\begin{algorithm}[t]
\caption{ALASCA}\label{alg:alasca}
\textbf{Require}: $\left\{ \mathbf{x}_{i}, \mathbf{y}_{i} \right\}$, $1 \leq i \leq N$ \\
\textbf{Require}: $\mathcal{L}$ \Comment{Loss function for existing LNL methods}\\
\textbf{Require}: $\left\{ \Theta_{k} \right\}_{k=0}^{K}$ \Comment{Parameters for main and ACs} \\
\textbf{Require}: $\beta$, $\tau$ \Comment{EMA weight and temperature of ALASCA} \\
\textbf{Require}: $\lambda$ \Comment{Coefficient for power of regularization}
\textbf{Output}: $\Theta_{0}$
\begin{algorithmic}[1] 
\State $\textbf{t}_{i} \leftarrow \mathbf{0}_{\text{N}}$. \Comment{Initialize EMA confidence}
\For{each minibatch B}
\State $\textbf{t}_{i} \leftarrow \beta \textbf{t}_{i} + (1-\beta) \mathbf{f}_{\Theta_{0}}(\mathbf{x}_{i})$ \Comment{EMA (Averaging)}
\State $\tilde{\alpha}(\mathbf{x}_{i}) \leftarrow 1 - \mathcal{S}(\mathbf{t}_{i} / \tau)$ \Comment{EMA (Sharpening)}
\State loss $\leftarrow$ $-\frac{1}{\vert B \vert} \sum_{i=1}^{\vert B \vert} \mathcal{L} \left( \mathbf{f}_{\Theta_{0}}(\mathbf{x}_{i}), \mathbf{y}_{i} \right)$
\State \textcolor{white}{loss $\leftarrow$} $+\frac{\lambda}{\vert B \vert} \sum_{k=1}^{K} \sum_{i=1}^{\vert B \vert} \ell^{ALS}_{\tilde{\alpha}(\mathbf{x}_{i})} \left( \mathbf{f}_{\Theta_{k}}(\mathbf{x}_{i}), \mathbf{y}_{i} \right)$
\State \textcolor{gray}{/* Compute loss by using Eq. \eqref{eq:als} */}
\State update $\Theta_{0}$ and $\left\{ \Theta_{k} \right\}_{k=1}^{K}$ using SGD
\EndFor
\end{algorithmic}
\end{algorithm}

\begin{table*}[t]
\centering
\caption{Test accuracies (\%) on CIFAR-10/-100 under different noise types and fractions for noise-robust loss and sample-selection approaches. The results for symmetric and asymmetric noise of all baseline methods were taken from \citet{kim2021fine}. Instance-dependent noise results are reported by our re-implementation based on official codes. The average accuracies over three trials are reported. The best results sharing the noisy fraction and method are highlighted in bold.}
\scriptsize{ 
\begin{tabular}{l|ccccc|cccc} \toprule
Dataset      & \multicolumn{5}{c}{CIFAR-10}     & \multicolumn{4}{|c}{CIFAR-100}       \\ \midrule
Noisy Type   & \multicolumn{3}{c}{Symm.} & Asym. & Inst. & \multicolumn{3}{c}{Symm.} & Asym. \\ \midrule
Noise Ratio  & 20\% & 50\% & 80\% & 40\% & 40\% & 20\% & 50\% & 80\% & 40\%  \\ \midrule \midrule
Standard  & 87.0 $\pm$ 0.1 & 78.2 $\pm$ 0.8 & 53.8 $\pm$ 1.0 & 85.0 $\pm$ 0.1 & 74.1 $\pm$ 2.9 & 58.7 $\pm$ 0.3 & 42.5 $\pm$ 0.3 & 18.1 $\pm$ 0.8 & 42.7 $\pm$ 0.6 \\
\; \textbf{+ ALASCA} & \textbf{92.2 $\pm$ 0.2} & \textbf{88.0 $\pm$ 0.3} & \textbf{70.3 $\pm$ 0.4} & \textbf{90.3 $\pm$ 0.3} & \textbf{81.4 $\pm$ 0.3} & \textbf{70.6 $\pm$ 0.2} & \textbf{59.7 $\pm$ 0.4} & \textbf{26.1 $\pm$ 0.9} & \textbf{59.0 $\pm$ 0.6} \\
\midrule
GCE  & 89.8 $\pm$ 0.2 & 86.5 $\pm$ 0.2 & 64.1 $\pm$ 1.4 & 76.7 $\pm$ 0.6 & 55.7 $\pm$ 0.2 & 66.8 $\pm$ 0.4 & 57.3 $\pm$ 0.3 & 29.2 $\pm$ 0.7 & 47.2 $\pm$ 1.2 \\
\; \textbf{+ ALASCA} & \textbf{92.3 $\pm$ 0.3} & \textbf{88.4 $\pm$ 0.3} & \textbf{73.7 $\pm$ 1.5} & \textbf{90.2 $\pm$ 0.1} & \textbf{73.5 $\pm$ 1.2} & \textbf{70.8 $\pm$ 0.5} & \textbf{60.1 $\pm$ 0.4} & \textbf{32.2 $\pm$ 0.5} & \textbf{57.3 $\pm$ 0.5}  \\ 
\midrule
SCE & 89.8 $\pm$ 0.3 & 84.7 $\pm$ 0.3 & 68.1 $\pm$ 0.8 & 82.5 $\pm$ 0.5 & 71.4 $\pm$ 1.3 & 70.4 $\pm$ 0.1 & 48.8 $\pm$ 1.3 & 25.9 $\pm$ 0.4 & 48.4 $\pm$ 0.9 \\
\; \textbf{+ ALASCA} & \textbf{92.8 $\pm$ 0.1} & \textbf{88.5 $\pm$ 0.1} & \textbf{71.7 $\pm$ 0.8} & \textbf{89.4 $\pm$ 0.2} & \textbf{78.0 $\pm$ 0.6} & \textbf{71.4 $\pm$ 0.2} & \textbf{61.3 $\pm$ 0.6} & \textbf{28.7 $\pm$ 0.6} & \textbf{57.3 $\pm$ 0.7} \\
\midrule
ELR & 91.2 $\pm$ 0.1 & 88.2 $\pm$ 0.1 & 72.9 $\pm$ 0.6 & 90.1 $\pm$ 0.5 & 79.8 $\pm$ 0.2 & 74.2 $\pm$ 0.2 & 59.1 $\pm$ 0.8 & 29.8 $\pm$ 0.6 & 73.3 $\pm$ 0.4 \\
\; \textbf{+ ALASCA} & \textbf{92.3 $\pm$ 0.2} & \textbf{89.5 $\pm$ 0.3} & \textbf{74.2 $\pm$ 1.2} & \textbf{90.4 $\pm$ 0.2} & \textbf{82.2 $\pm$ 0.5} & \textbf{74.8 $\pm$ 0.1} & \textbf{63.6 $\pm$ 0.6} & \textbf{34.4 $\pm$ 0.4} & \textbf{73.9 $\pm$ 0.5} \\
\midrule
Co-teaching & 89.3 $\pm$ 0.3 & 83.3 $\pm$ 0.6 & 66.3 $\pm$ 1.5 & 88.4 $\pm$ 2.8 & 70.5 $\pm$ 0.5 & 63.4 $\pm$ 0.0 & 49.1 $\pm$ 0.4 & 20.5 $\pm$ 1.3 & 47.7 $\pm$ 1.2 \\
\; \textbf{+ ALASCA} & \textbf{93.4 $\pm$ 0.1} & \textbf{90.1 $\pm$ 1.5} & \textbf{71.1 $\pm$ 1.2} & \textbf{91.2 $\pm$ 0.1} & \textbf{76.8 $\pm$ 0.4} & \textbf{75.5 $\pm$ 0.3} & \textbf{68.1 $\pm$ 0.4} & \textbf{42.2 $\pm$ 1.2} & \textbf{64.7 $\pm$ 0.4} \\ 
\midrule
CRUST & 89.4 $\pm$ 0.2 & 87.0 $\pm$ 0.1 & 64.8 $\pm$ 1.1 & 82.4 $\pm$ 0.0 & 64.7 $\pm$ 2.1 & 69.3 $\pm$ 0.2 & 62.3 $\pm$ 0.2 & 21.7 $\pm$ 0.7 & 56.1 $\pm$ 0.5 \\
\; \textbf{+ ALASCA} & \textbf{92.3 $\pm$ 0.2} & \textbf{87.6 $\pm$ 0.2} & \textbf{71.5 $\pm$ 1.5} & \textbf{90.2 $\pm$ 0.2} & \textbf{71.6 $\pm$ 1.1} &  \textbf{70.4 $\pm$ 0.1} & \textbf{64.1 $\pm$ 0.9} & \textbf{25.5 $\pm$ 0.7} & \textbf{58.3 $\pm$ 0.5} \\
\midrule
FINE & 91.0 $\pm$ 0.1 & 87.3 $\pm$ 0.2 & 69.4 $\pm$ 1.1 & 89.5 $\pm$ 0.1 & 82.4 $\pm$ 0.5 & 70.3 $\pm$ 0.2 & 64.2 $\pm$ 0.5 & 25.6 $\pm$ 1.2 & 61.7 $\pm$ 1.0 \\
\; \textbf{+ ALASCA} & \textbf{92.1 $\pm$ 0.1} & \textbf{88.0 $\pm$ 0.1} & \textbf{70.6 $\pm$ 0.9} & \textbf{90.3 $\pm$ 0.2} & \textbf{84.3 $\pm$ 1.2} & \textbf{70.9 $\pm$ 0.3} & \textbf{65.8 $\pm$ 0.2} & \textbf{29.4 $\pm$ 1.5} & \textbf{63.5 $\pm$ 0.7} \\ 
\bottomrule
\end{tabular}
}
\label{tab:cifar-main}
\end{table*}

\paragraph{Use of Auxiliary Classifiers.} Recent studies have emphasized LR of intermediate layers. \citet{wei2019data, wei2019improved} showed that LR for all intermediate layers improves generalization of DNNs. \citet{sokolic2017robust} and \citet{cao2021heteroskedastic} showed similar results in data-limited and distribution shift setups, respectively. Inspired by these works, we use LS as an intermediate LR by utilizing an AC commonly used in deep learning. While ACs work well on various domains such as self-distillation\,\cite{zhang2019your} and class-imbalance\,\cite{lee2021abc}, our work first proposes using ACs with noisy labels based on theoretical motivation.
We compare our method with \citet{zhang2019your} in Appendix~\ref{sec:byot}.
We apply bottleneck\,\cite{howard2017mobilenets} architecture as ACs and attach them at the end of all blocks of backbone by our results in Section~\ref{sec:4.3}.
By applying ACs, we can encourage greater LR to noisy instances and weaker LR to clean instances with only a small additional computation. Furthermore, this approach has two additional advantages.
\citet{wei2021smooth} showed that the benefits of LS disappear in a high-level noise, as applying LS tends to over-smooth the estimated posterior of the main classifier.
However, LS with ACs does not affect the main classifier but effectively regularizes the Lipschitzness of intermediate layers. 
In Figure~\ref{fig:auxiliary}, the Jacobian norm is effectively reduced regardless of which classifier ALS is applied to, but the performances of ALS on ACs are higher than that of the main classifier.
Moreover, as we use multiple classifiers, we obtain more robust predictions from the ensemble effect during inference.
In Algorithm~\ref{alg:alasca}, we denote the parameters of the main classifier and ACs as $\Theta_{0}$ and $\left\{ \Theta_{k} \right\}_{k=1}^{K}$, where $K$ is number of AC. We further denote outputs of the $k$-th classifier as $\mathbf{f}_{\Theta_{k}}(\cdot)$.

\paragraph{Use of EMA Confidence.} As shown in Figure~\ref{fig:ema}, instantaneous confidence suffers from high variance across training epochs and is inaccurate for differentiating regularization power between clean and noisy instances. Incorrect regularization power caused by such instability leads to performance degradation. Hence, we use EMA along the epochs to compute confidence and effectively obtain the appropriate regularization power. In SSL, this weight averaging approach has been proposed to mitigate confirmation bias \cite{liu2020early}. The computation procedure of EMA confidence is as follows. (1) To reduce variance and enhance stability, we conduct EMA on output values (Line 3 in Algorithm~\ref{alg:alasca}). (2) Because the averaged outputs are over-smooth, which causes weak regularization for noisy examples and strong regularization for clean examples, we sharpen the EMA logits by dividing sharpen temperature $\tau$ (Line 4 in Algorithm~\ref{alg:alasca}). We observe that regularization powers on clean and noisy examples are clearly distinguished and become stable after using EMA confidence, as shown in Figure~\ref{fig:ema}.

\section{Experiments}
We design experiments to answer the following questions: %
\begin{itemize}
    \item Can ALASCA improve existing LNL methods, such as noise-robust loss functions and sample-selection methods for both synthetic and real-world datasets? (Section~\ref{sec:4.1}\,\&\,\ref{sec:4.4})
    \item How effective is ALACSA in improving the robustness of the feature extractor? (Section~\ref{sec:4.2})
    \item How do the architecture and position of ACs affect the performance and efficiency? \& Which component is important to the performance in ALASCA? (Section~\ref{sec:4.3})
\end{itemize}

\begin{figure*}[t]
    \hspace*{\fill}
    \begin{subfigure}[b]{0.27\textwidth} 
    \centering
    \includegraphics[width=\linewidth]{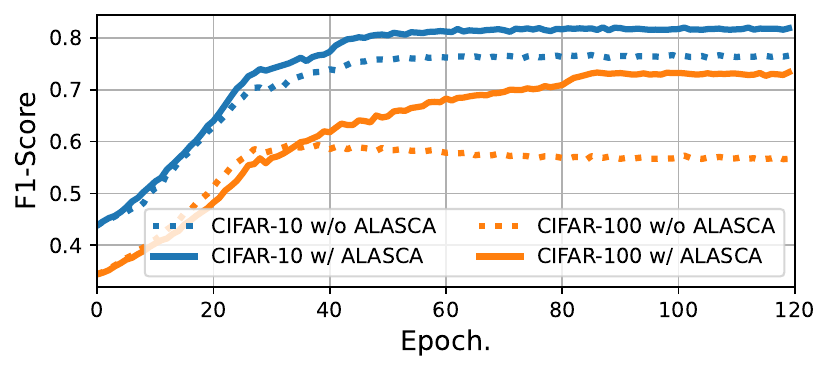}
    \caption{Co-teaching (F1-Score)}
    \end{subfigure}
    \hfill
    \begin{subfigure}[b]{0.27\textwidth} 
    \centering
    \includegraphics[width=\linewidth]{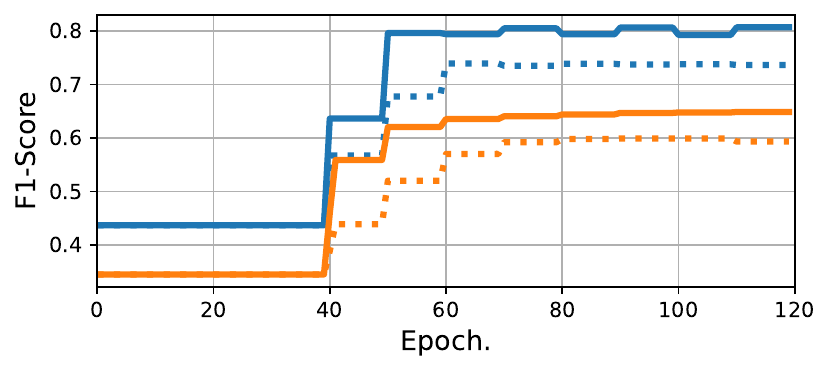}
    \caption{FINE (F1-Score)}
    \end{subfigure}
    \hfill
    \begin{subfigure}[b]{0.27\textwidth} 
    \centering
    \includegraphics[width=\linewidth]{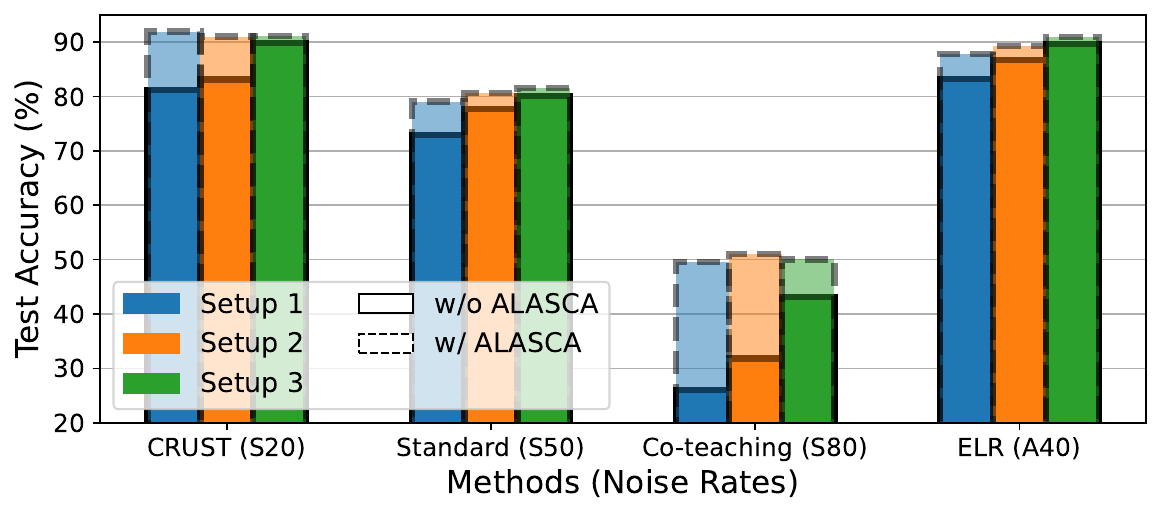}
    \caption{Hyperparameter Selection}\label{fig:hyperparam}
    \end{subfigure}
    \hspace*{\fill}
    \caption{Comparison of F1-scores of (a) Co-teaching; (b) FINE with and without ALASCA on CIFAR-10 and CIFAR-100 under 80\% of symmetric noise. (c) Comparison of test accuracies (\%) along different hyperparameter settings for various LNL methods. Through the results, we verify that ALASCA enhances the robustness of feature extractors.}\label{fig:quality}
\end{figure*}

\subsection{Experimental Setup and Results on CIFAR}\label{sec:4.1}
\paragraph{Setup.} We inject uniform randomness into a fraction of labels for symmetric noise and flip labels to specific classes for asymmetric noise by following \citet{kim2021fine}. 
To set up instance-dependent noise, we follow the noise generation of \citet{cheng2020learning}. 
We use the architectures of backbone network and hyperparameter settings for all baseline experiments following \citet{kim2021fine}. 
We set $\beta$, $\tau$, and $\lambda$ as 0.7, 1/3, and 2.0, respectively. The detailed experimental setup is described in Appendix~\ref{app:4.0} and \ref{app:4.1}.
To verify the superiority of our method, we combine ALASCA with various existing LNL methods (noise-robust loss functions and sample-selection methods) and identify that ours consistently improves the generalization in the presence of noisy data. Furthermore, we perform additional experiments incorporating semi-supervised approaches with ALASCA. Appendix~\ref{sec:semi} provides detailed description and results for the SSL approaches.

\paragraph{Noise-Robust Loss Functions.} Noise-robust loss functions aim to achieve high performances for unseen clean data despite the presence of noisy labels in the training data. We combine our proposed method with three loss functions: (1) standard CE (Standard); (2) generalized CE (GCE; \citealt{zhang2018generalized}), which can be seen as a generalization of the mean absolute error and standard CE; (3) symmetric CE (SCE; \citealt{wang2019symmetric}), which is the weighted sum of CE and reverse version of CE; and (4) early learning regularization (ELR; \citealt{liu2020early}) which uses a regularization term that incorporates target probabilities from the model output. We observe that ALASCA improves generalization when applied to the noise-robust loss function of Table~\ref{tab:cifar-main}.

\paragraph{Sample-Selection Methods.} Sample-selection methods, which select clean sample candidates from the training dataset, are a popular direction in LNL. 
We combine ALASCA with the following sample-selection approaches: (1) Co-teaching \cite{han2018co}, which utilizes two networks, extracts subsets of examples with small losses from each network, and trains each network with subsets of examples filtered by another network; (2) CRUST \cite{mirzasoleiman2020coresets}, which selects a subset of small weight gradient instances; and (3) FINE \cite{kim2021fine}, which selects instances whose penultimate vector highly correlates with the class-representative eigenvector. Table~\ref{tab:cifar-main} shows the performance increase of ALASCA with different sample-selection methods on various label noise.

\subsection{Quality of Feature Extractors}\label{sec:4.2}
Many LNL methods use posterior information with undesired bias from the corrupted networks under noisy labels, which can lead to sub-optimal performances. 
However, if the feature extractor is robustly trained under label noise, we can employ unbiased posterior information and achieve better generalization. 
To validate the effectiveness of ALASCA in terms of improving the robustness of the feature extractor and mitigating undesirable biases, we conduct exploratory experiments: (1) comparison of quality for subsets of sample-selection approaches and (2) robustness to hyperparameter selection of existing LNL methods. 

\begin{table*}[t]
\centering
\caption{Comparison of test accuracy (\%) on Clothing1M dataset. Results for baselines are obtained from \citet{liu2020early} and \citet{kim2021fine}. A-Coteaching and A-ELR+ denote the methods combining ALASCA with Co-teaching and ELR+.}
\footnotesize{ 
\begin{tabular}{l|ccccccc|ccc} \toprule
Method  & Standard & GCE & SCE & Co-teaching & FINE & DivideMix & ELR+ & ALASCA & A-Coteaching & A-ELR+ \\ \midrule
Accuracy  & 68.94 & 69.75 & 71.02 & 70.15 & 72.91 & 74.76 & 74.81 & 73.78 & 74.20 & \textbf{74.92} \\ \bottomrule
\end{tabular}
}
\label{tab:clothing1m}
\end{table*}

\begin{table}[t]
\centering
\caption{Comparison of accuracy\,(\%), GPU memory, and computation time. For HAR and CDR, we apply their official code. Without and with * denote attaching ACs at the end of only the third and all residual blocks of the backbone network, respectively.}
\scriptsize{ 
\begin{tabular}{l|ccc|c|c} \toprule
Dataset      & \multicolumn{3}{c|}{CIFAR-10} & \multicolumn{2}{c}{Efficiency} \\ \midrule 
Noisy Type & Symm. & Asym. & Inst. & Mem. & Time \\ \midrule \midrule
Standard  & 81.9 $\pm$ 0.3 &  85.0 $\pm$ 0.1 & 74.1 $\pm$ 2.9 & $\times$ 1.0 & $\times$ 1.0 \\
HAR  & 88.6 $\pm$ 0.1 & 87.5 $\pm$ 0.2 & 67.6 $\pm$ 0.7 & $\times$ 2.0 & $\times$ 5.2 \\
CDR  & 87.0 $\pm$ 0.2 & 88.7 $\pm$ 0.3 & 75.8 $\pm$ 1.5 & $\times$ 1.3 & $\times$ 6.5 \\ \midrule
\multicolumn{6}{l}{ALASCA (Different architecture and position of ACs)} \\ \midrule
MLP & 88.5 $\pm$ 0.3 & 89.4 $\pm$ 0.2 & 76.3 $\pm$ 0.3 & \underline{$\times$ 1.1} & \textbf{$\times$ 1.0} \\
MLP$^{*}$ & 89.3 $\pm$ 0.1 & 89.7 $\pm$ 0.1 & \underline{80.8 $\pm$ 0.3} & $\times$ 1.3 & $\times$ 1.2 \\
Bottleneck & 89.2 $\pm$ 0.1 & 89.8 $\pm$ 0.1 & 77.9 $\pm$ 0.4 & \textbf{$\times$ 1.0} & \textbf{$\times$ 1.0} \\
Bottleneck$^{*}$ & \textbf{90.1 $\pm$ 0.1} & \textbf{90.3 $\pm$ 0.3} & \textbf{81.4 $\pm$ 0.3} & \underline{$\times$ 1.1} & $\times$ 1.2 \\
Residual & 89.0 $\pm$ 0.2 & 89.6 $\pm$ 0.3 & 77.1 $\pm$ 0.3 & \underline{$\times$ 1.1} & \underline{$\times$ 1.1} \\
Residual$^{*}$ & \underline{89.8 $\pm$ 0.2} & \underline{90.1 $\pm$ 0.5} & 77.8 $\pm$ 0.5 & $\times$ 1.3 & $\times$ 1.3 \\
\bottomrule
\end{tabular}}
\label{tab:regularize-methods}
\end{table}

\paragraph{Quality of Sample Selection.} To verify that our proposed method robustly trains the feature extractor on label noise, we compute the F1-score for all training epochs to evaluate noise sample filtering in sample-selection methods for various symmetric and asymmetric label noise.
We compare the quality of sample selection with and without ALASCA on the sample-selection approaches: Co-teaching and FINE. In Figure~\ref{fig:quality}, the F1-scores of sample selection with ALASCA are consistently higher on various label noise than the baseline. 
Although baseline approaches use different criteria to filter noisy instances\,(Co-teaching with loss values; FINE with cosine similarities in the penultimate layer), we observe that ALASCA improves the quality of all subsets and is effective in training robust feature extractors.

\paragraph{Robust to Hyperparameter Selection.} Existing LNL methods have large performance differences depending on their hyperparameters, and, if the hyperparameter value is improperly selected, the performance is provably lower than that of standard training.
However, the value of the optimal hyperparameters depends on the network architecture and dataset. 
We combine ALASCA and existing LNL methods with various hyperparameter settings: (1) Standard with different weight decay factors; (2) ELR with different regularization coefficients; (3) Co-teaching along different warmup epochs; and (4) CRUST with different coreset sizes. 
The detailed experiment setup and results are in Appendix~\ref{app:4.2} and Figure~\ref{fig:hyperparam}, respectively.
While the baseline performances vary depending on the hyperparameters, performances with ALASCA are robust even with different hyperparameters.

\subsection{Ablation Studies}\label{sec:4.3}
To obtain further intuition on ALASCA, we conduct an ablation study on each component of our method.
We design three experiments: (1) compare existing methods in terms of performance and efficiency; (2) investigate performance tendency along the structure of the AC; and (3) component analysis to understand the influence of each component.

\paragraph{Efficiency of ALASCA.} We compare our implicit regularization framework with the following regularization-based approaches: (1) HAR \cite{cao2021heteroskedastic}, which explicitly and adaptively regularizes the Jacobian matrix norm of data points in higher-uncertainty, lower-density regions more heavily and (2) CDR \cite{xia2020robust}, which identifies and regularizes non-critical parameters that tend to fit noisy labels and cannot generalize well. 
These methods are similar to our proposed method for regularizing the intermediate layer, but are explicit regularization methods that are computationally expensive to find noisy data or parameters.
Table~\ref{tab:regularize-methods} shows that ALASCA consistently outperforms competing regularization-based methods for various noise types and provides efficient training in terms of computational memory and training time.

\paragraph{Effect of Auxiliary Classifiers.} We further compare the performance of ALASCA with different architectures (2 layers MLP, residual block;\,\citealt{he2016deep}, bottleneck;\,\citealt{howard2017mobilenets}) and positions of ACs. Table~\ref{tab:regularize-methods} shows how the performance of ALASCA changes according to the architecture and position of ACs. We observe that using several ACs effectively regularizes the intermediate layer, resulting in improved generalization. This result supports our motivations for using ACs to regularize intermediate layers and thereby enhance the robustness against label noise. 
Furthermore, the architecture of ACs does not affect performance much. 
Among the three architectures, bottleneck classifiers achieve competitive performance with the smallest additional computation costs. From these results, we apply the bottleneck block as the AC for all experiments.

\paragraph{Component Analysis.} Since our ALASCA is composed of three parts: (1) ALS; (2) ACs; and (3) EMA confidence, we perform a component analysis to understand which component is important for training robust feature extractors.
Our experiments are conducted on CIFAR-10 under various noise distributions. Table~\ref{tab:component} summarizes that each component is indeed effective, as the performance improves with each addition of a component. However, the most important factor for high performance is the combination of ALS and ACs, which enable effective LR on intermediate layers. Moreover, we verify that the performance of ALASCA improves using an ensemble of predictions from the main and ACs as we mentioned in Section~\ref{sec:3.2}.

\begin{table}[t]
\caption{Component analysis on each component of our proposed methods. ALASCA$^{*}$ denotes that result from ensemble of all classifiers during inference phase and the bold numbers indicate the best result.}
\centering
\scriptsize{
\begin{tabular}{l|cccc} \toprule
  & Symm. 20 & Symm. 50 & Asym. 40 & Inst. 40 \\ \midrule
ULS & 87.3 $\pm$ 0.5 & 69.5 $\pm$ 1.1 & 78.7 $\pm$ 0.5 & 59.2 $\pm$ 0.6 \\
+ AC & 91.1 $\pm$ 0.1 & 84.7 $\pm$ 0.2 & 87.2 $\pm$ 0.2 & 77.0 $\pm$ 1.0 \\ \midrule
ALS & 90.0 $\pm$ 0.3 & 82.0 $\pm$ 0.7 & 85.1 $\pm$ 0.2 & 74.8 $\pm$ 1.3 \\
+ AC & 91.8 $\pm$ 0.1 & 86.6 $\pm$ 0.2 & 90.0 $\pm$ 0.3 & 78.8 $\pm$ 0.6 \\
+ EMA & 90.1 $\pm$ 0.2 & 82.9 $\pm$ 0.8 & 86.9 $\pm$ 1.1 & 75.1 $\pm$ 2.1 \\ \midrule
ALASCA & 92.2 $\pm$ 0.2 & 88.0 $\pm$ 0.3 & 90.3 $\pm$ 0.3 & 81.4 $\pm$ 0.4 \\
ALASCA$^{*}$ & \textbf{92.8 $\pm$ 0.1} & \textbf{89.1 $\pm$ 0.4} & \textbf{90.6 $\pm$ 0.2} & \textbf{82.5 $\pm$ 0.9} \\
\bottomrule
\end{tabular}}\label{tab:component}
\end{table}

\subsection{Results on Real-world Datasets}\label{sec:4.4}
Clothing1M~\cite{xiao2015learning} contains one million clothing images obtained from online shopping websites with 14 classes and estimated noise level of 38.5\%~\cite{song2019does}. We apply ResNet50, which is widely used in previous studies \cite{liu2020early, kim2021comparing} on the Clothing1M dataset. Table~\ref{tab:clothing1m} compares ALASCA to the SOTA methods on the Clothing1M dataset. ALASCA achieves a competitive performance with the SOTA baseline methods although using only a single network with lower computational costs. Furthermore, we observe that ELR+ with ALASCA (A-ELR+) realizes a new SOTA performance and verify that our proposed method also works well on real-world datasets.
Additionally, we apply ALASCA on the (mini) WebVision dataset, a famous real-world dataset with label noise, and obtain similar results to those on Clothing1M. We report the detailed results in Appendix~\ref{sec:webvision}.

\section{Conclusion}
In this paper, we provide a theoretical analysis that LS encourages LR, and build upon the resulting insights to propose an effective and practical framework, ALASCA.
Based on the resulting theoretical motivation, we combine ALS, AC, and EMA confidence to efficiently enable adaptive LR.
We experimentally show that ALASCA enhances the robustness of feature extractors and improves the performance of existing LNL methods on benchmark-simulated and real-world datasets.
In future work, we believe that our approach will arise interest in designing a novel regularization strategy for feature extractors.

\section*{Acknowledgements}
This work was supported by Center for Applied Research in Artificial Intelligence (CARAI) grant funded by Defense Acquisition Program Administration (DAPA) and Agency for Defense Development (ADD) (UD190031RD).

\bibliography{aaai23}

\begin{thebibliography}{55}
\providecommand{\natexlab}[1]{#1}

\bibitem[{Arpit et~al.(2017)Arpit, Jastrzebski, Ballas, Krueger, Bengio,
  Kanwal, Maharaj, Fischer, Courville, Bengio et~al.}]{arpit2017closer}
Arpit, D.; Jastrzebski, S.; Ballas, N.; Krueger, D.; Bengio, E.; Kanwal, M.~S.;
  Maharaj, T.; Fischer, A.; Courville, A.; Bengio, Y.; et~al. 2017.
\newblock A closer look at memorization in deep networks.
\newblock In \emph{International Conference on Machine Learning}, 233--242.
  PMLR.

\bibitem[{Bekker and Goldberger(2016)}]{bekker2016training}
Bekker, A.~J.; and Goldberger, J. 2016.
\newblock Training deep neural-networks based on unreliable labels.
\newblock In \emph{2016 IEEE International Conference on Acoustics, Speech and
  Signal Processing (ICASSP)}, 2682--2686. IEEE.

\bibitem[{Berthelot et~al.(2019)Berthelot, Carlini, Goodfellow, Papernot,
  Oliver, and Raffel}]{berthelot2019mixmatch}
Berthelot, D.; Carlini, N.; Goodfellow, I.; Papernot, N.; Oliver, A.; and
  Raffel, C. 2019.
\newblock Mixmatch: A holistic approach to semi-supervised learning.
\newblock \emph{arXiv preprint arXiv:1905.02249}.

\bibitem[{Cao et~al.(2021)Cao, Chen, Lu, Arechiga, Gaidon, and
  Ma}]{cao2021heteroskedastic}
Cao, K.; Chen, Y.; Lu, J.; Arechiga, N.; Gaidon, A.; and Ma, T. 2021.
\newblock Heteroskedastic and Imbalanced Deep Learning with Adaptive
  Regularization.
\newblock In \emph{International Conference on Learning Representations}.

\bibitem[{Chen et~al.(2020)Chen, Kornblith, Norouzi, and
  Hinton}]{chen2020simple}
Chen, T.; Kornblith, S.; Norouzi, M.; and Hinton, G. 2020.
\newblock A simple framework for contrastive learning of visual
  representations.
\newblock In \emph{International conference on machine learning}, 1597--1607.
  PMLR.

\bibitem[{Chen and Gupta(2015)}]{chen2015webly}
Chen, X.; and Gupta, A. 2015.
\newblock Webly supervised learning of convolutional networks.
\newblock In \emph{Proceedings of the IEEE International Conference on Computer
  Vision}, 1431--1439.

\bibitem[{Cheng et~al.(2020)Cheng, Zhu, Li, Gong, Sun, and
  Liu}]{cheng2020learning}
Cheng, H.; Zhu, Z.; Li, X.; Gong, Y.; Sun, X.; and Liu, Y. 2020.
\newblock Learning with instance-dependent label noise: A sample sieve
  approach.
\newblock \emph{arXiv preprint arXiv:2010.02347}.

\bibitem[{Cheng et~al.(2021)Cheng, Zhu, Sun, and Liu}]{cheng2021demystifying}
Cheng, H.; Zhu, Z.; Sun, X.; and Liu, Y. 2021.
\newblock Demystifying how self-supervised features improve training from noisy
  labels.
\newblock \emph{arXiv preprint arXiv:2110.09022}.

\bibitem[{Chorowski and Jaitly(2017)}]{chorowski2017towardsBD}
Chorowski, J.; and Jaitly, N. 2017.
\newblock Towards Better Decoding and Language Model Integration in Sequence to
  Sequence Models.
\newblock In \emph{INTERSPEECH}, 523--527.

\bibitem[{Filiposka, Djuric, and ElMaraghy(2014)}]{filiposka2014complexity}
Filiposka, M.~Z.; Djuric, A.~M.; and ElMaraghy, W. 2014.
\newblock Complexity analysis for calculating the Jacobian matrix of 6DOF
  reconfigurable machines.
\newblock \emph{Procedia CIRP}, 17: 218--223.

\bibitem[{Finlay et~al.(2018)Finlay, Calder, Abbasi, and
  Oberman}]{finlay2018lipschitz}
Finlay, C.; Calder, J.; Abbasi, B.; and Oberman, A. 2018.
\newblock Lipschitz regularized deep neural networks generalize and are
  adversarially robust.
\newblock \emph{arXiv preprint arXiv:1808.09540}.

\bibitem[{Ghosh, Kumar, and Sastry(2017)}]{ghosh2017robust}
Ghosh, A.; Kumar, H.; and Sastry, P. 2017.
\newblock Robust loss functions under label noise for deep neural networks.
\newblock In \emph{Proceedings of the AAAI Conference on Artificial
  Intelligence}, volume~31.

\bibitem[{Ghoshal et~al.(2021)Ghoshal, Chen, Gupta, Zettlemoyer, and
  Mehdad}]{ghoshal2021learningBS}
Ghoshal, A.; Chen, X.; Gupta, S.; Zettlemoyer, L.; and Mehdad, Y. 2021.
\newblock Learning Better Structured Representations Using Low-rank Adaptive
  Label Smoothing.
\newblock In \emph{International Conference on Learning Representations}.

\bibitem[{Gouk et~al.(2021)Gouk, Frank, Pfahringer, and
  Cree}]{gouk2021regularisation}
Gouk, H.; Frank, E.; Pfahringer, B.; and Cree, M.~J. 2021.
\newblock Regularisation of neural networks by enforcing lipschitz continuity.
\newblock \emph{Machine Learning}, 110(2): 393--416.

\bibitem[{Han et~al.(2018)Han, Yao, Yu, Niu, Xu, Hu, Tsang, and
  Sugiyama}]{han2018co}
Han, B.; Yao, Q.; Yu, X.; Niu, G.; Xu, M.; Hu, W.; Tsang, I.; and Sugiyama, M.
  2018.
\newblock Co-teaching: Robust training of deep neural networks with extremely
  noisy labels.
\newblock \emph{arXiv preprint arXiv:1804.06872}.

\bibitem[{He et~al.(2016)He, Zhang, Ren, and Sun}]{he2016deep}
He, K.; Zhang, X.; Ren, S.; and Sun, J. 2016.
\newblock Deep residual learning for image recognition.
\newblock In \emph{Proceedings of the IEEE conference on computer vision and
  pattern recognition}, 770--778.

\bibitem[{Howard et~al.(2017)Howard, Zhu, Chen, Kalenichenko, Wang, Weyand,
  Andreetto, and Adam}]{howard2017mobilenets}
Howard, A.~G.; Zhu, M.; Chen, B.; Kalenichenko, D.; Wang, W.; Weyand, T.;
  Andreetto, M.; and Adam, H. 2017.
\newblock Mobilenets: Efficient convolutional neural networks for mobile vision
  applications.
\newblock \emph{arXiv preprint arXiv:1704.04861}.

\bibitem[{Kim et~al.(2021{\natexlab{a}})Kim, Ko, Cho, Choi, and
  Yun}]{kim2021fine}
Kim, T.; Ko, J.; Cho, S.; Choi, J.; and Yun, S.-Y. 2021{\natexlab{a}}.
\newblock FINE Samples for Learning with Noisy Labels.
\newblock arXiv:2102.11628.

\bibitem[{Kim et~al.(2021{\natexlab{b}})Kim, Oh, Kim, Cho, and
  Yun}]{kim2021comparing}
Kim, T.; Oh, J.; Kim, N.; Cho, S.; and Yun, S.-Y. 2021{\natexlab{b}}.
\newblock Comparing Kullback-Leibler Divergence and Mean Squared Error Loss in
  Knowledge Distillation.
\newblock \emph{arXiv preprint arXiv:2105.08919}.

\bibitem[{Lee, Shin, and Kim(2021)}]{lee2021abc}
Lee, H.; Shin, S.; and Kim, H. 2021.
\newblock ABC: Auxiliary Balanced Classifier for Class-imbalanced
  Semi-supervised Learning.
\newblock \emph{Advances in Neural Information Processing Systems}, 34.

\bibitem[{Lee et~al.(2018)Lee, He, Zhang, and Yang}]{lee2018cleannet}
Lee, K.-H.; He, X.; Zhang, L.; and Yang, L. 2018.
\newblock Cleannet: Transfer learning for scalable image classifier training
  with label noise.
\newblock In \emph{Proceedings of the IEEE Conference on Computer Vision and
  Pattern Recognition}, 5447--5456.

\bibitem[{Li, Socher, and Hoi(2020)}]{li2020dividemix}
Li, J.; Socher, R.; and Hoi, S.~C. 2020.
\newblock Dividemix: Learning with noisy labels as semi-supervised learning.
\newblock \emph{arXiv preprint arXiv:2002.07394}.

\bibitem[{Li, Dasarathy, and Berisha(2020)}]{li2020regularization}
Li, W.; Dasarathy, G.; and Berisha, V. 2020.
\newblock Regularization via Structural Label Smoothing.
\newblock In \emph{Proceedings of the Twenty Third International Conference on
  Artificial Intelligence and Statistics}, 1453--1463. PMLR.

\bibitem[{Li et~al.(2017)Li, Wang, Li, Agustsson, and
  Van~Gool}]{li2017webvision}
Li, W.; Wang, L.; Li, W.; Agustsson, E.; and Van~Gool, L. 2017.
\newblock Webvision database: Visual learning and understanding from web data.
\newblock \emph{arXiv preprint arXiv:1708.02862}.

\bibitem[{Liu et~al.(2020)Liu, Niles-Weed, Razavian, and
  Fernandez-Granda}]{liu2020early}
Liu, S.; Niles-Weed, J.; Razavian, N.; and Fernandez-Granda, C. 2020.
\newblock Early-learning regularization prevents memorization of noisy labels.
\newblock \emph{arXiv preprint arXiv:2007.00151}.

\bibitem[{Lukasik et~al.(2020)Lukasik, Bhojanapalli, Menon, and
  Kumar}]{lukasik2020does}
Lukasik, M.; Bhojanapalli, S.; Menon, A.; and Kumar, S. 2020.
\newblock Does label smoothing mitigate label noise?
\newblock In \emph{International Conference on Machine Learning}, 6448--6458.
  PMLR.

\bibitem[{Ma et~al.(2020)Ma, Huang, Wang, Romano, Erfani, and
  Bailey}]{ma2020normalized}
Ma, X.; Huang, H.; Wang, Y.; Romano, S.; Erfani, S.; and Bailey, J. 2020.
\newblock Normalized loss functions for deep learning with noisy labels.
\newblock In \emph{International Conference on Machine Learning}, 6543--6553.
  PMLR.

\bibitem[{Malach and Shalev-Shwartz(2017)}]{malach2017decoupling}
Malach, E.; and Shalev-Shwartz, S. 2017.
\newblock Decoupling" when to update" from" how to update".
\newblock \emph{arXiv preprint arXiv:1706.02613}.

\bibitem[{Mirzasoleiman, Cao, and Leskovec(2020)}]{mirzasoleiman2020coresets}
Mirzasoleiman, B.; Cao, K.; and Leskovec, J. 2020.
\newblock Coresets for robust training of neural networks against noisy labels.
\newblock \emph{arXiv preprint arXiv:2011.07451}.

\bibitem[{M{\"u}ller, Kornblith, and Hinton(2019)}]{muller2019does}
M{\"u}ller, R.; Kornblith, S.; and Hinton, G. 2019.
\newblock When does label smoothing help?
\newblock \emph{arXiv preprint arXiv:1906.02629}.

\bibitem[{Nesser et~al.(2021)Nesser, Jacob, Maasakkers, Scarpelli, Sulprizio,
  Zhang, and Rycroft}]{nesser2021reduced}
Nesser, H.; Jacob, D.~J.; Maasakkers, J.~D.; Scarpelli, T.~R.; Sulprizio,
  M.~P.; Zhang, Y.; and Rycroft, C.~H. 2021.
\newblock Reduced-cost construction of Jacobian matrices for high-resolution
  inversions of satellite observations of atmospheric composition.
\newblock \emph{Atmospheric Measurement Techniques}, 14(8): 5521--5534.

\bibitem[{Nettleton, Orriols-Puig, and Fornells(2010)}]{nettleton2010study}
Nettleton, D.~F.; Orriols-Puig, A.; and Fornells, A. 2010.
\newblock A study of the effect of different types of noise on the precision of
  supervised learning techniques.
\newblock \emph{Artificial intelligence review}, 33(4): 275--306.

\bibitem[{Nguyen et~al.(2020)Nguyen, Mummadi, Ngo, Nguyen, Beggel, and
  Brox}]{nguyen2020self}
Nguyen, D.~T.; Mummadi, C.~K.; Ngo, T. P.~N.; Nguyen, T. H.~P.; Beggel, L.; and
  Brox, T. 2020.
\newblock SELF: Learning to Filter Noisy Labels with Self-Ensembling.
\newblock In \emph{International Conference on Learning Representations}.

\bibitem[{Pereyra et~al.(2017)Pereyra, Tucker, Chorowski, Kaiser, and
  Hinton}]{pereyra2017regularizing}
Pereyra, G.; Tucker, G.; Chorowski, J.; Kaiser, Å.; and Hinton, G. 2017.
\newblock Regularizing Neural Networks by Penalizing Confident Output
  Distributions.
\newblock \emph{arXiv preprint arXiv:1701.06548}.

\bibitem[{Sokoli{\'c} et~al.(2017)Sokoli{\'c}, Giryes, Sapiro, and
  Rodrigues}]{sokolic2017robust}
Sokoli{\'c}, J.; Giryes, R.; Sapiro, G.; and Rodrigues, M.~R. 2017.
\newblock Robust large margin deep neural networks.
\newblock \emph{IEEE Transactions on Signal Processing}, 65(16): 4265--4280.

\bibitem[{Song et~al.(2019)Song, Kim, Park, and Lee}]{song2019does}
Song, H.; Kim, M.; Park, D.; and Lee, J.-G. 2019.
\newblock How does Early Stopping Help Generalization against Label Noise?
\newblock \emph{arXiv preprint arXiv:1911.08059}.

\bibitem[{Szegedy et~al.(2016)Szegedy, Vanhoucke, Ioffe, Shlens, and
  Wojna}]{szegedy2016rethinking}
Szegedy, C.; Vanhoucke, V.; Ioffe, S.; Shlens, J.; and Wojna, Z. 2016.
\newblock Rethinking the inception architecture for computer vision.
\newblock In \emph{Proceedings of the IEEE conference on computer vision and
  pattern recognition}, 2818--2826.

\bibitem[{Tibshirani(2014)}]{tibshirani2014adaptive}
Tibshirani, R.~J. 2014.
\newblock Adaptive piecewise polynomial estimation via trend filtering.
\newblock \emph{The Annals of Statistics}, 42(1): 285--323.

\bibitem[{Wang, Du, and Shen(2013)}]{wang2013smoothing}
Wang, X.; Du, P.; and Shen, J. 2013.
\newblock Smoothing splines with varying smoothing parameter.
\newblock \emph{Biometrika}, 100(4): 955--970.

\bibitem[{Wang et~al.(2019)Wang, Ma, Chen, Luo, Yi, and
  Bailey}]{wang2019symmetric}
Wang, Y.; Ma, X.; Chen, Z.; Luo, Y.; Yi, J.; and Bailey, J. 2019.
\newblock Symmetric cross entropy for robust learning with noisy labels.
\newblock In \emph{Proceedings of the IEEE/CVF International Conference on
  Computer Vision}, 322--330.

\bibitem[{Wei and Ma(2019{\natexlab{a}})}]{wei2019data}
Wei, C.; and Ma, T. 2019{\natexlab{a}}.
\newblock Data-dependent sample complexity of deep neural networks via
  lipschitz augmentation.
\newblock \emph{arXiv preprint arXiv:1905.03684}.

\bibitem[{Wei and Ma(2019{\natexlab{b}})}]{wei2019improved}
Wei, C.; and Ma, T. 2019{\natexlab{b}}.
\newblock Improved sample complexities for deep networks and robust
  classification via an all-layer margin.
\newblock \emph{arXiv preprint arXiv:1910.04284}.

\bibitem[{Wei et~al.(2021{\natexlab{a}})Wei, Liu, Liu, Niu, and
  Liu}]{wei2021understanding}
Wei, J.; Liu, H.; Liu, T.; Niu, G.; and Liu, Y. 2021{\natexlab{a}}.
\newblock Understanding Generalized Label Smoothing when Learning with Noisy
  Labels.
\newblock \emph{arXiv preprint arXiv:2106.04149}.

\bibitem[{Wei et~al.(2021{\natexlab{b}})Wei, Liu, Liu, Niu, Sugiyama, and
  Liu}]{wei2021smooth}
Wei, J.; Liu, H.; Liu, T.; Niu, G.; Sugiyama, M.; and Liu, Y.
  2021{\natexlab{b}}.
\newblock To Smooth or Not? When Label Smoothing Meets Noisy Labels.
\newblock \emph{Learning}, 1(1): e1.

\bibitem[{Xia et~al.(2021)Xia, Liu, Han, Gong, Wang, Ge, and
  Chang}]{xia2020robust}
Xia, X.; Liu, T.; Han, B.; Gong, C.; Wang, N.; Ge, Z.; and Chang, Y. 2021.
\newblock Robust early-learning: Hindering the memorization of noisy labels.
\newblock In \emph{International Conference on Learning Representations}.

\bibitem[{Xiao et~al.(2015)Xiao, Xia, Yang, Huang, and Wang}]{xiao2015learning}
Xiao, T.; Xia, T.; Yang, Y.; Huang, C.; and Wang, X. 2015.
\newblock Learning from massive noisy labeled data for image classification.
\newblock In \emph{Proceedings of the IEEE conference on computer vision and
  pattern recognition}, 2691--2699.

\bibitem[{Yi and Wu(2019)}]{yi2019probabilistic}
Yi, K.; and Wu, J. 2019.
\newblock Probabilistic end-to-end noise correction for learning with noisy
  labels.
\newblock In \emph{Proceedings of the IEEE/CVF Conference on Computer Vision
  and Pattern Recognition}, 7017--7025.

\bibitem[{Yu et~al.(2019)Yu, Han, Yao, Niu, Tsang, and Sugiyama}]{yu2019does}
Yu, X.; Han, B.; Yao, J.; Niu, G.; Tsang, I.; and Sugiyama, M. 2019.
\newblock How does disagreement help generalization against label corruption?
\newblock In \emph{International Conference on Machine Learning}, 7164--7173.
  PMLR.

\bibitem[{Zhang et~al.(2017{\natexlab{a}})Zhang, Bengio, Hardt, Recht, and
  Vinyals}]{zhang2017understanding}
Zhang, C.; Bengio, S.; Hardt, M.; Recht, B.; and Vinyals, O.
  2017{\natexlab{a}}.
\newblock Understanding deep learning requires rethinking generalization.
\newblock arXiv:1611.03530.

\bibitem[{Zhang et~al.(2017{\natexlab{b}})Zhang, Cisse, Dauphin, and
  Lopez-Paz}]{zhang2017mixup}
Zhang, H.; Cisse, M.; Dauphin, Y.~N.; and Lopez-Paz, D. 2017{\natexlab{b}}.
\newblock mixup: Beyond empirical risk minimization.
\newblock \emph{arXiv preprint arXiv:1710.09412}.

\bibitem[{Zhang and Yao(2020)}]{zhang2020decoupling}
Zhang, H.; and Yao, Q. 2020.
\newblock Decoupling Representation and Classifier for Noisy Label Learning.
\newblock \emph{arXiv preprint arXiv:2011.08145}.

\bibitem[{Zhang et~al.(2019)Zhang, Song, Gao, Chen, Bao, and
  Ma}]{zhang2019your}
Zhang, L.; Song, J.; Gao, A.; Chen, J.; Bao, C.; and Ma, K. 2019.
\newblock Be your own teacher: Improve the performance of convolutional neural
  networks via self distillation.
\newblock In \emph{Proceedings of the IEEE/CVF International Conference on
  Computer Vision}, 3713--3722.

\bibitem[{Zhang and Sabuncu(2018)}]{zhang2018generalized}
Zhang, Z.; and Sabuncu, M.~R. 2018.
\newblock Generalized cross entropy loss for training deep neural networks with
  noisy labels.
\newblock In \emph{32nd Conference on Neural Information Processing Systems
  (NeurIPS)}.

\bibitem[{Zheltonozhskii et~al.(2022)Zheltonozhskii, Baskin, Mendelson,
  Bronstein, and Litany}]{zheltonozhskii2022contrast}
Zheltonozhskii, E.; Baskin, C.; Mendelson, A.; Bronstein, A.~M.; and Litany, O.
  2022.
\newblock Contrast to divide: Self-supervised pre-training for learning with
  noisy labels.
\newblock In \emph{Proceedings of the IEEE/CVF Winter Conference on
  Applications of Computer Vision}, 1657--1667.

\bibitem[{Zhou et~al.(2021)Zhou, Liu, Jiang, Gao, and Ji}]{zhou2021asymmetric}
Zhou, X.; Liu, X.; Jiang, J.; Gao, X.; and Ji, X. 2021.
\newblock Asymmetric Loss Functions for Learning with Noisy Labels.
\newblock \emph{arXiv preprint arXiv:2106.03110}.

\end{thebibliography}

\clearpage
\appendix
\onecolumn
\begin{center}
    \textbf{\Large{Appendix:\\}}
    \textbf{\large{Robust Training with Adaptive Label Smoothing via Auxiliary Classifier under Label Noise}}
\end{center}

\section{Theoretical Analysis}
\subsection{Description for Remark~\ref{method:cao}}\label{appendix:thm1} 
In this section, we provide a detailed explanation of \citet{cao2021heteroskedastic}. We begin by defining the notations in \citet{cao2021heteroskedastic}. 
In this paper, the authors deal with a one-dimensional binary logistic regression problem, where the model can be expressed as
$$\text{Pr}(Y=y|X=x) = \frac{1}{1+\exp(-yf^*(x))},$$
where $\mathcal{X}=[0,1]$, $\mathcal{Y}=\{-1,1\}$, and $f^*$ is the ground-truth function. The objective function containing Lipschitz regularization is as follows:
$$\hat{f} = \text{argmin}_f \;\, \frac{1}{n} \sum_{i=1}^n \ell(f(x_i),y_i) + \lambda \int_0^1 \rho(x)(f'(x))^2 dx,$$
where $\ell$ is the cross-entropy loss and $\rho(x)$ is a data-dependent smoothing parameter for adaptive Lipschitz regularization. Let $q(x)$ be the density of $X$, and $I(x)$ be the fisher information matrix conditioned on $X=x$. Lastly, the mean squared error (MSE) of the test set $\{ (x_i, y_i) \}_{i=1}^n$ is defined as
$$MSE(\hat{f}):=\mathbb{E}_{\{(x_i,y_i)\}_{i=1}^n} \int_0^1 (\hat{f}(t)- f^*(t))^2 dt $$
Now, the following main theorem of \citet{cao2021heteroskedastic} provides an analytical formula for the asymptotic MSE on the test set.\\

\noindent
\textit{\textbf{Theorem 1. \cite{cao2021heteroskedastic}} Assume that $f^*,q,I \in W_2^2 = \{ f' \text{ is absolute continuous and } f'' \in L^2[0,1]\}$. Let $r(t)=-1/(q(t)I(t))$ and $L_0 = \int_{-\infty}^{\infty} \frac14 \exp(-2|t|)dt$. If we choose $\lambda = C_0 n^{-2/5}$ for some constant $C_0>0$, the asymptotic mean squared error is}
$$\lim_{n\rightarrow \infty} MSE(\hat{f}) = C_n \int_{0}^{1} \lambda^2r^2(t)\left [ \frac{\mathrm{d} }{\mathrm{d} t} (\rho(t)(f^*)'(t))\right ]^2 + L_0r(t)^{1/2}\rho(t)^{-1/2}dt$$
\textit{in probability, where $C_n$ is a scalar that only depends on $n$.}\\

Using this formula, they choose an optimal smoothing parameter $\rho(x)$ that minimizes the asymptotic MSE. 
However, not only do we not know $q(x)$ and $f^*$, but exact computation is unachievable in practice. Thus, the authors make the following two assumptions:
1) The data can be divided into $k$ intervals and $\rho(t)$, $q(t)$, and $I(t)$ is constant on each of the interval, 2) $\frac{\mathrm{d}^2 }{\mathrm{d} t^2} f^*(t)$ is close to a constant on the entire space. 
With these simplifications, the authors derive that the asymptotic MSE is minimized when $\rho \propto q^{3/5} I^{3/5}$. We obtain an approximate optimal choice of the smoothing parameter $\rho(x)$. Note that $I(x)= \text{Var}(Y|X=x)$ and $I(x)$ represents the label uncertainty of $x$.

\subsection{Theoretical Guarantees for Motivations}\label{appendix:thm2}
This section provides the detailed proof for Theorem~\ref{method:theorem1} and Theorem~\ref{method:theorem2}.
Before diving into the main proof, we provide a lemma showing twice-differentiability and convexity of the LS regularization term.

\begin{lemma}\label{lemma1}
    Define a function $\Phi := \Omega \circ \mathbf{g}$, i.e. $\Phi (\mathbf{z}) = \Omega(\mathbf{g}(\mathbf{z}))= \Omega(\mathbf{f})$. Then, $\Phi (\mathbf{z})$ is twice differentiable and also convex. 
\end{lemma}
\begin{proof}
Let $Q$ denote the dimension of $\mathbf{z}$, i.e. $\mathbf{z} \in \mathbb{R}^{Q}$. That is, $\Phi$ is a function from $\mathbb{R}^{Q}$ to $\mathbb{R}$. The gradient and Hessian of $\Phi(\mathbf{z})$ is as follows:
\begin{equation}
\begin{split}
    \nabla \Phi(\mathbf{z})= \frac{\partial \Phi(\mathbf{z})}{\partial \mathbf{z}} &= \frac{\partial}{\partial \mathbf{z}} \left[ L \cdot \log \left[ \sum_{i=1}^{L} e^{\left< \mathbf{W}_{i}, \mathbf{z} \right>} \right] - \sum_{i=1}^{L} \left< \mathbf{W}_{i}, \mathbf{z} \right> \right] \\
    &= L \cdot \frac{\sum_{i=1}^{L} e^{\left< \mathbf{W}_{i}, \mathbf{z} \right>} \, \mathbf{W}_{i}}{\sum_{i=1}^{L} e^{\left< \mathbf{W}_{i}, \mathbf{z} \right>}} - \sum_{i=1}^{L} \mathbf{W}_{i},
\end{split}
\end{equation}
\begin{equation}
\begin{split}
    \nabla^2 \Phi(\mathbf{z}) 
    &= L \cdot \frac{\sum_{i=1}^{L} e^{\left< \mathbf{W}_{i}, \mathbf{z} \right>} \, \mathbf{W}_{i} \mathbf{W}_{i}^{\intercal} \cdot \sum_{i=1}^{L} e^{\left< \mathbf{W}_{i}, \mathbf{z} \right>} - \left ( \sum_{i=1}^{L} e^{\left< \mathbf{W}_{i}, \mathbf{z} \right>} \, \mathbf{W}_{i}\right )\left ( \sum_{i=1}^{L} e^{\left< \mathbf{W}_{i}, \mathbf{z} \right>} \, \mathbf{W}_{i}^{\intercal}\right) }{\left(\sum_{i=1}^{L} e^{\left< \mathbf{W}_{i}, \mathbf{z} \right>}\right)^2}.
\end{split}
\end{equation}
Now, let's show that $\nabla^2 \Phi(\mathbf{z})$ is positive semi-definite for all $z\in\mathbb{R}^{Q}$. For any $\mathbf{v} \in\mathbb{R}^{Q}$,
\begin{equation}
\begin{split}
     & \qquad \mathbf{v}^{\intercal}  \left( \sum_{i=1}^{L} e^{\left< \mathbf{W}_{i}, \mathbf{z} \right>} \, \mathbf{W}_{i} \mathbf{W}_{i}^{\intercal} \cdot \sum_{i=1}^{L} e^{\left< \mathbf{W}_{i}, \mathbf{z} \right>} - \left ( \sum_{i=1}^{L} e^{\left< \mathbf{W}_{i}, \mathbf{z} \right>} \, \mathbf{W}_{i}\right )\left ( \sum_{i=1}^{L} e^{\left< \mathbf{W}_{i}, \mathbf{z} \right>} \, \mathbf{W}_{i}^{\intercal}\right)\right) \mathbf{v} \\
    & = \sum_{i=1}^{L} e^{\left< \mathbf{W}_{i}, \mathbf{z} \right>} (\mathbf{W}_{i}^{\intercal} \mathbf{v} )^2 \cdot \sum_{i=1}^{L} e^{\left< \mathbf{W}_{i}, \mathbf{z} \right>} - \left ( \sum_{i=1}^{L} e^{\left< \mathbf{W}_{i}, \mathbf{z} \right>} \, \mathbf{W}_{i}^{\intercal} \mathbf{v} \right)^2 \geq 0
\end{split}
\end{equation}
by Cauchy-Schwarz Inequality. Thus, $\mathbf{v}^{\intercal} \nabla^2 \Phi(\mathbf{z}) \mathbf{v}\geq 0$ for all $\mathbf{v} \in\mathbb{R}^{Q}$ and this proves that $\nabla^2 \Phi(\mathbf{z})$ is positive semi-definite for all $\mathbf{z}\in\mathbb{R}^{Q}$ and $\Phi (\mathbf{z})$ is convex.
\end{proof}

\noindent\textbf{Theorem \ref{method:theorem1}.} $\mathbf{h}=\mathbf{0}$ is the minimizer of $\Omega \circ \mathbf{g}$. \, If Assumption \ref{thm:assumption1} holds, $\mathbf{h}=\mathbf{0}$ is the unique minimizer.
\begin{proof}
We specify the domain and codomain of each function: $\mathbf{g}: \mathbb{R}^{Q} \rightarrow \mathbb{R}^{L}, \mathbf{h}: \mathcal{X} \rightarrow \mathbb{R}^{Q}$ where $\mathcal{X} \subset \mathbb{R}^{D}$. Then,
\begin{equation}
    \left.\frac{\partial \Omega(\mathbf{f})}{\partial \mathbf{z}} \right\vert_{\mathbf{z}=\mathbf{0}} = \left. L \cdot \frac{\sum_{i=1}^{L} e^{\left< \mathbf{W}_{i}, \mathbf{z} \right>} \, \mathbf{W}_{i}}{\sum_{i=1}^{L} e^{\left< \mathbf{W}_{i}, \mathbf{z} \right>}} - \sum_{i=1}^{L} \mathbf{W}_{i} \right\vert_{\mathbf{z}=\mathbf{0}} = \mathbf{0}.
\end{equation}
Since $\Omega(\mathbf{f})$ is convex respect to $\mathbf{z}$ by Lemma \ref{lemma1}, $\mathbf{z}=\mathbf{h}(\mathbf{x})=\mathbf{0}$ is the global minimizer of $\Omega(\mathbf{f})$. Now, by Assumption \ref{thm:assumption1}, $\mathbf{W}_2-\mathbf{W}_1, \mathbf{W}_3-\mathbf{W}_1, \cdots, \mathbf{W}_L-\mathbf{W}_1$ are linearly independent and is a basis for $\mathbb{R}^{Q}$. We express $\frac{\partial \Omega(\mathbf{f})}{\partial \mathbf{z}}$ with the basis as follows:
\begin{equation}
\begin{split}
    \frac{\partial \Omega(\mathbf{f})}{\partial \mathbf{z}} &= L \cdot \frac{\sum_{i=1}^{L} e^{\left< \mathbf{W}_{i}, \mathbf{z} \right>} \, \mathbf{W}_{i}}{\sum_{i=1}^{L} e^{\left< \mathbf{W}_{i}, \mathbf{z} \right>}} - \sum_{i=1}^{L} \mathbf{W}_{i} \\
    &= L \cdot \frac{\sum_{i=1}^{L} e^{\left< \mathbf{W}_{i} - \mathbf{W}_{1}, \mathbf{z} \right>} \, \mathbf{W}_{i}}{\sum_{i=1}^{L} e^{\left< \mathbf{W}_{i} - \mathbf{W}_{1}, \mathbf{z} \right>}} - \sum_{i=1}^{L} \mathbf{W}_{i} \\
    & = L \cdot \frac{\sum_{i=1}^{L} e^{\left< \mathbf{W}_{i} - \mathbf{W}_{1}, \mathbf{z} \right>} \, (\mathbf{W}_{i} - \mathbf{W}_{1})}{\sum_{i=1}^{L} e^{\left< \mathbf{W}_{i} - \mathbf{W}_{1}, \mathbf{z} \right>}} - \sum_{i=1}^{L} (\mathbf{W}_{i} - \mathbf{W}_{1}).
\end{split}
\end{equation}
Since $\{\mathbf{W}_{j} - \mathbf{W}_{1}: 2 \leq j \leq L \}$ forms a basis, we have 
\begin{equation}\label{eq10}
\begin{split}
    \frac{\partial \Omega(\mathbf{f})}{\partial \mathbf{z}} = \mathbf{0} & \iff L \cdot \frac{ e^{\left< \mathbf{W}_{j} - \mathbf{W}_{1}, \mathbf{z} \right>} }{\sum_{i=1}^{L} e^{\left< \mathbf{W}_{i} - \mathbf{W}_{1}, \mathbf{z} \right>}} - 1 = 0 \quad \text{ for all } \;1\leq j \leq L \\
    & \iff \left< \mathbf{W}_{j} - \mathbf{W}_{1}, \mathbf{z} \right> = 0 \quad \text{ for all } \;2\leq j \leq L.
\end{split}
\end{equation}
We find the optimal value $\mathbf{z}$ which minimizes $\Omega(\mathbf{f})$ by solving the $\frac{\partial \Omega(\mathbf{f})}{\partial \mathbf{z}} = \mathbf{0}$. By Equation \ref{eq10}, we have
\begin{equation}
   \begin{pmatrix}
    \mathbf{W}_{2} - \mathbf{W}_{1}\\ 
    \mathbf{W}_{3} - \mathbf{W}_{1}\\ 
    \vdots \\ 
    \mathbf{W}_{L} - \mathbf{W}_{1}
    \end{pmatrix}^{\intercal} \mathbf{z} = \mathbf{0}.
\end{equation}
Again by Assumption \ref{thm:assumption1}, the left multiplied matrix is a full rank square matrix. Hence, we obtain $\frac{\partial \Omega(\mathbf{f})}{\partial \mathbf{z}} = \mathbf{0} \Leftrightarrow \mathbf{z} = \mathbf{0}$, which implies that $\mathbf{z}=\mathbf{h}(\mathbf{x})=\mathbf{0}$ is the unique minimizer of $\Omega(\mathbf{f})$.
\end{proof}

\noindent\textbf{Theorem \ref{method:theorem2}.} Consider a bounded feature space $\mathcal{X}$ and suppose that Assumption \ref{thm:assumption2} is satisfied. If $\mathbf{h}(\mathbf{x}_n)=\mathbf{0}$ for $1\leq n \leq N$, $\left \|  \mathbf{J_f}(\mathbf{x}_n)\right \|_F \rightarrow \mathbf{0}$ as $N \rightarrow \infty$ holds for $1\leq n \leq N$.
\begin{proof}
For all $1 \leq n \leq N$ and $1 \leq q \leq Q$, because we have $\mathbf{h}(\mathbf{x}_n) = \mathbf{h}(\mathbf{x}_{n+1}) = \mathbf{0}, \;\exists\; \mathbf{c}_n^q \in \overline{\mathbf{x}_n \mathbf{x}_{n+1}} \;\,s.t. \;\nabla h_q(\mathbf{c}_n^q) = \mathbf{0}$ by the Mean Value Theorem. Then, using Assumption \ref{thm:assumption2}, we obtain
\begin{equation}
    \left \|\nabla h_q(\mathbf{x}_{n}) \right \| \leq L_h \left \| \mathbf{x}_{n} - \mathbf{c}_n^q \right \| \leq L_h \left \| \mathbf{x}_{n} - \mathbf{x}_{n+1} \right \|.
\end{equation}
Therefore, 
\begin{equation}
    \left \|  \mathbf{J_h}(\mathbf{x}_n)\right \|_F = \sqrt{\sum_{q=1}^{Q} \left \|\nabla h_q(\mathbf{x}_{n}) \right \|^2} \leq L_h \sqrt{Q} \left \| \mathbf{x}_{n} - \mathbf{x}_{n+1} \right \|
\end{equation}
and for sufficiently large $N$ and any $\epsilon>0$, we can have the training set satisfy $\left \| \mathbf{x}_{n} - \mathbf{x}_{n+1} \right \| \leq \frac{\epsilon}{L_h \sqrt{Q}}$ since $\mathcal{X}$ is bounded. This leads to 
\begin{equation}
    \left \|  \mathbf{J_h}(\mathbf{x}_n)\right \|_F \leq \epsilon
\end{equation}
which implies that $\left \|  \mathbf{J_h}(\mathbf{x}_n)\right \|_F \rightarrow \mathbf{0}$ as $N$ goes to $\infty$. Finally, since $\mathbf{f}$ is a composite function of $\mathbf{g}$ and $\mathbf{h}$, $\left \|  \mathbf{J_f}(\mathbf{x}_n)\right \|_F \rightarrow \mathbf{0}$ as $N \rightarrow \infty$ also holds. Hence, we may conclude that the label smoothing regularizer encourages Lipschitz regularization of deep neural network functions.
\end{proof}

\section{Full Version of Related Works}\label{app:related}
Numerous studies have been conducted on classification tasks with noisy labels.
Although we cannot explain all related works in the main paper due to page limitations, here we fully describe existing related works on LNL methods.

\paragraph{Noise-Cleansing base Approaches.} 
Noise-Cleansing methods identify and filter out the noisy information and only exploit the clean data.
\citet{malach2017decoupling} proposed a decoupling strategy that trains two networks simultaneously and updates the model only using the instances that have different predictions from the two networks.
Co-teaching \cite{han2018co}, which employs two networks, exploits subsets of low-loss instances in each network, and trains each network using the subsets of instances from the other network. 
\citet{mirzasoleiman2020coresets} presented an approach that samples subsets of clean cases that produce an essentially low-rank Jacobian matrix and demonstrated that gradient descent applied to the selections prevents overfitting noisy labels.
\citet{kim2021fine} proposed a detecting framework, termed FINE which utilizes the high-order topological information of data in latent space by using eigen decomposition of their covariance matrix.
Recently, \citet{li2020dividemix} modeled the per-sample loss distribution and divided the samples into a labeled set containing clean samples and an unlabeled set with noisy samples, and apply MixMatch \cite{berthelot2019mixmatch}, which is one of the most famous semi-supervised methods.

\paragraph{Robust Loss Functions.} Some methods focus on providing loss functions that are stable and can be shown to be tolerant of label noise. \citet{ghosh2017robust} theoretically proved that the mean absolute error (MAE) might be robust against noisy labels while commonly used CE loss is not. \citet{zhang2018generalized} argued that MAE performed poorly with DNNs and proposed a GCE loss function, which can be seen as a generalization of MAE and CE. \citet{wang2019symmetric} introduced the reverse version of the cross-entropy term (RCE) and suggested the SCE loss function as a weighted sum of CE and RCE. \citet{ma2020normalized} showed that a normalized version of any loss would be robust to noisy labels. Recently, \citet{zhou2021asymmetric} argued that the symmetric condition of existing robust loss functions is overly restrictive and proposed asymmetric loss function families which alleviate symmetric conditions while maintaining noise-tolerant of noise robust loss function.

\paragraph{Robust Architecture.} Architectural changes in a certain network may mitigate the negative effect of noisy examples. Building a new network with noisy adaptation layers is a possible way to increase performance. Generally, a noise adaptation layer is located after the softmax layer and is not utilized in the inference phase. \citet{chen2015webly} utilized the concept of curriculum learning. They build a confusion matrix with small loss examples and initialize the weight parameters of the noise adaption layer with the confusion matrix. EM algorithm was implemented in \citet{bekker2016training} estimating true labels in E-step and updating parameters with back-propagation in M-step.
\section{Implementation Setup for Experimental Results}
All experiments on CIFAR used a single Nvidia GeForce 2080Ti GPU. For real-world dataset experiments (Clothing1M, WebVision), we apply a single NVIDIA A100 40GB for all individual experiments. We apply normalization and simple data augmentation techniques (random crop and horizontal flip) to the training sets of all datasets. We additionally describe the architecture of bottleneck network \cite{howard2017mobilenets} that we use as ACs in Figure~\ref{fig:bottleneck}. We attach this AC between the repetitive blocks of all backbone networks.

\begin{figure}[ht]
    \centering
    \includegraphics[width=0.5\linewidth]{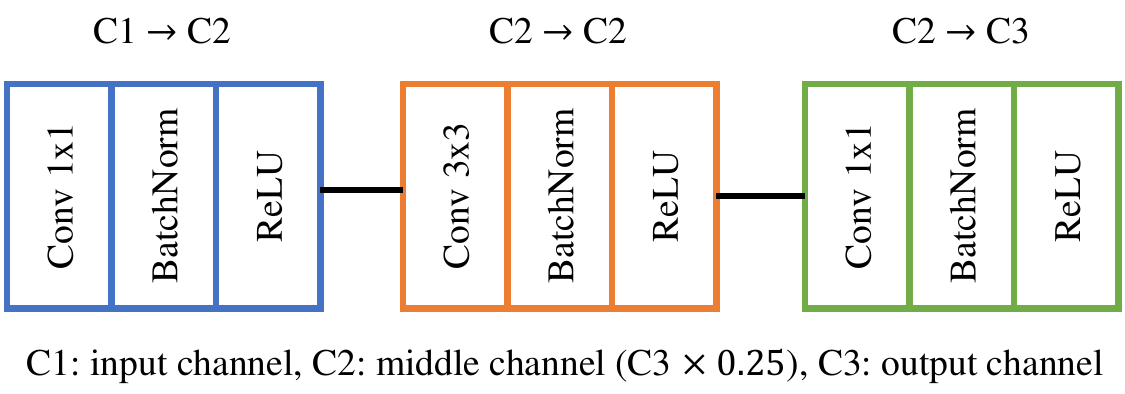}
    \caption{Architecture of bottleneck networks for ACs in our ALASCA}
    \label{fig:bottleneck}
\end{figure}

\subsection{Dataset Generation for Synthetic Datasets}\label{app:4.0}
\paragraph{Noisy CIFAR-10/-100.} By following \citet{kim2021fine}, we synthetically generate the symmetric noise, asymmetric noise. We inject uniform randomness into a fraction of labels for symmetric noise. Furthermore, we generate asymmetric noise by mapping TRUCK\,$\rightarrow$\,AUTOMOBILE, BIRD\,$\rightarrow$\,AIRPLANE, DEER\,$\rightarrow$\,HORSE, CAT\,$\rightarrow$\,DOG for CIFAR-10. For CIFAR-100, we divide the entire dataset into 20 super-classes of size five and generate asymmetric noise by changing each class to the next class within each super-class. Furthermore, based on \citet{cheng2020learning}, we generate instance-based noise and compare the performance with existing LNL methods. With the defined noise rate (the global flipping rate) $\epsilon$, we first sample their flip rates from a truncated normal distribution $N(\epsilon, 0.1^{2}, [0, 1])$, where $[0, 1]$ indicates the range of the truncated normal distribution. Then, we sample parameters $W$ from the standard normal distribution for generating instance-dependent label noise. The size of $W$ is $S \times K$, where $S$ denotes the length of each feature. 

\paragraph{Clothing1M.} Clothing1M contains one million clothing images obtained from online shopping websites with 14 classes. 
The dataset offers 50k, 14k, and 10k clean data points that have been confirmed for use in training, validation, and testing. We employ a randomly picked pseudo-balanced subset with 120k images as a training set rather than these 50k clean training data. On the 10k clean testing dataset, we compute the classification accuracy for evaluation. For A-Coteaching, we set the noise ratio of 38.5\% by following \citet{song2019does}. For each epoch, we sample 2000 mini-batches from the training data ensuring that the classes of the noisy labels are balanced.

\subsection{Experimental Setup for Section~\ref{sec:4.1} and \ref{sec:4.4}}\label{app:4.1}
We use the architecture and hyperparameter settings for all baseline experiments following the setup of \citet{liu2020early} and \citet{kim2021fine}, except for the SSL approaches. For SSL approaches, we follow the original works \cite{li2020dividemix, liu2020early}. As we mentioned, we apply $\beta, \tau$ and $\lambda$ as 0.7, 1/3, and 2.0. We further conduct sensitivity analysis of hyperparameters in our method in Appendix~\ref{sec:sensitivity}, showing that our method is reasonably robust to hyperparameter settings.

\paragraph{Noise-Robust Loss Functions.} We conduct experiments with CE, GCE, SCE, and ELR as mentioned in Section~\ref{sec:4.1}. We follow all experiment settings presented in the \citet{liu2020early}. We use ResNet34 models and trained them using a standard Pytorch SGD optimizer with momentum of 0.9. We set a batch size of 128 for all experiments. We utilize weight decay of 0.001 and set the initial learning rate as 0.02, and reduce it by a factor of 100 after 40 and 80 epochs for CIFAR-10 (total 120 epochs) and after 80 and 120 epochs for CIFAR-100 (total 150 epochs).

\paragraph{Sample-Selection Methods.} In the experiments, we follow all experiment settings presented in the \citet{kim2021fine}. We use ResNet34 (or as the backbone with ALASCA) models and trained them using a standard Pytorch SGD optimizer with an initial learning rate of 0.02 and a momentum of 0.9. We set a batch size of 128 for all experiments and utilize a  weight decay of 0.001. Unlike the noise-robust loss functions, we trained all baseline with our settings and reported the results. For Co-teaching and FINE, we set the warm-up epochs as 30 and 40, respectively. For CRUST, we set the coreset size as 0.5 and do not apply Mixup \cite{zhang2017mixup} by following \citet{kim2021fine}.

\paragraph{Clothing1M.} For backbone network, we apply ResNet-50 pretrained on ImageNet dataset by following previous works \cite{liu2020early, kim2021comparing}. The model is trained with batch size 64 and initial learning rate $1\times10^{-3}$, and reduce it by a factor of 100 after 7 epochs (total 15 epochs). We use Pytorch SGD optimizer with weight decay of 0.001 and momentum of 0.9.

\subsection{Experimental Setup Section~\ref{sec:4.2}}\label{app:4.2}
In this section, we fully describe the experimental settings which contain noise levels, backbone network architectures, hyperparameter for the corresponding method, and hyperparameter values for Figure~\ref{fig:hyperparam}. The remaining experimental settings and hyperparameter values for our method are the same as Section~\ref{app:4.1}.

\paragraph{Standard Training.} \citet{yu2019does} argued that LS denoises label noise by encouraging weight-shrinkage of DNNs. Motivated by this view, we find that the performance of the network depends on the weight decay factor for training. We train the network with weight decay factor of $1\times10^{-4}$, $5\times10^{-4}$, and $1\times10^{-3}$. We apply VGG19 as backbone networks on CIFAR-10 with 60\% of symmetric noise.

\paragraph{ELR.} ELR \cite{liu2020early} uses two types of hyperparameters: the temporal ensembling parameter and the regularization coefficient which are denoted as $w_{1}$ and $w_{2}$ in Table~\ref{tab:hp-setup}, respectively. The original work of ELR performed hyperparameter tuning on the CIFAR datasets via grid search. The selected values are $w_{1}=0.7$ and $w_{2}=3.0$ for CIFAR-10 with symmetric noise, $w_{1}=0.9$ and $w_{2}=1$ for CIFAR-10 with asymmetric noise, and $w_{1}=0.9$ and $w_{2}=7$ for CIFAR-100. We train the network on CIFAR-10 under 40\% of asymmetric noise with the above three hyperparameter settings. We apply PreAct-ResNet18 as backbone networks on CIFAR-10 with 40\% of asymmetric noise.

\paragraph{Co-teaching.} The warm-up epochs for Co-teaching \cite{han2018co} are very important. If the number of warm-up epochs is too small, the network starts to filter the clean examples with underfitted network parameters. Moreover, the network is overfitted to noisy examples if the number of warm-up epochs is too large. While the original work reported the warm-up epochs as 30, here, we train the network with warm-up epochs 10, 20, and 30. We apply ResNet32 as backbone networks on CIFAR-10 with 80\% of symmetric noise.

\paragraph{CRUST.} The coreset size is important to CRUST since the network risks overfitting to noisy instances with large coresets and underfitting to clean examples with small coresets. The original paper of CRUST \cite{mirzasoleiman2020coresets} selected the coreset size as 50\% of the dataset size. In our experiments, we train the network with corset sizes of 30\%, 50\%, and 70\% of the dataset size. We apply WideResNet28-4 as backbone networks on CIFAR-10 with 20\% of symmetric noise.

\begin{table}[ht]
\caption{Description of noise rate, backbone network architecture, hyperparameter, and parameter values for each method (Standard, ELR, Co-teaching, CRUST) of experiments in Section~\ref{sec:4.2}.} \label{tab:hp-setup}
\centering
\scriptsize{
\begin{tabular}{l|cccc} \toprule
 & Standard & ELR & Co-teaching & CRUST \\ \midrule
Noise Rate & Symm. 60\% & Asym. 40\% & Symm. 80\% & Symm. 20\% \\
Network & VGG19 & PreAct-ResNet18 & ResNet32 & WideResNet28-4 \\
Hyperparameter & Weight decay & $w_{1}, w_{2}$ & Warm-up epochs & Coreset size \\
\multirow{3}{*}{Values} & $1\times10^{-4}$ (S1) & 0.7, 3.0 (S1) & 10 (S1) & 30\% (S1) \\
                        & $5\times10^{-4}$ (S2) & 0.9, 1.0 (S2) & 20 (S2) & 50\% (S2) \\
                        & $1\times10^{-3}$ (S3) & 0.9, 7.0 (S3) & 30 (S3) & 70\% (S3) \\
\bottomrule
\end{tabular}}
\end{table}

\section{Additional Analysis on the Experiments}

\subsection{Combination with Semi-supervised Approaches}\label{sec:semi}
SSL approaches \cite{li2020dividemix, liu2020early} fully utilize the label information of clean examples while they correct the information of noisy examples by regarding them as unlabeled data.
Recently, methods belonging to this category have shown the best performance among the various LNL methods, and these methods possibly train robust networks for even extremely high noise rates. However, pseudo labels generated from the corrupted networks might lead the network to sub-optimal.
Since most of the SSL approaches use Mixup \cite{zhang2017mixup} augmentation, we propose pseudo code for ALASCA with Mixup in Algorithm~\ref{alg:alasca_mixup}. We combine the ALASCA with existing semi-supervised approaches and compare the performance with baselines. All experimental settings are the same as the original works, and the hyperparameter values of ALASCA are the same with Appendix~\ref{app:4.1}.

\begin{algorithm}[ht]
\caption{ALASCA (Mixup version)}\label{alg:alasca_mixup}
\textbf{Require}: $\left\{ \mathbf{x}_{i}, \mathbf{y}_{i} \right\}$, $1 \leq i \leq N$ \\
\textbf{Require}: $\mathcal{L}$ \Comment{Loss function for existing LNL methods}\\
\textbf{Require}: $\left\{ \Theta_{k} \right\}_{k=0}^{K}$ \Comment{Parameters for main classifier and ACs} \\
\textbf{Require}: $\beta$, $\tau$ \Comment{EMA weight and temperature of ALASCA} \\
\textbf{Require}: $\lambda$ \Comment{Coefficient for power of regularization} \\
\textbf{Require}: $\gamma$ \Comment{Hyperparameter of Beta distribution for sampling Mixup weight} \\
\textbf{Output}: $\Theta_{0}$
\begin{algorithmic}[1] 
\State $\textbf{t}_{i} \leftarrow \mathbf{0}_{\text{N}}$. \Comment{Initialize EMA confidence}
\For{each minibatch B}
\State $\textbf{t}_{i} \leftarrow \beta \textbf{t}_{i} + (1-\beta) \mathbf{f}_{\Theta_{0}}(\mathbf{x}_{i})$ \Comment{EMA (Averaging)}
\State $\tilde{\alpha}(\mathbf{x}_{i}) \leftarrow 1 - \mathcal{S}(\mathbf{t}_{i} / \tau)$ \Comment{EMA (Sharpening)}

\State $\zeta \sim Beta(\gamma, \gamma)$ \Comment{Sampling Mixup weight}
\State $\tilde{\mathbf{x}}_{i} = \zeta \mathbf{x}_{i} + (1 - \zeta) \mathbf{x}_{j}$\;\;$\mathbf{x}_{i}, \mathbf{x}_{j} \in$ B \Comment{Mix the inputs}
\State $\tilde{\mathbf{y}}_{i} = \zeta \mathbf{y}_{i} + (1 - \zeta) \mathbf{y}_{j}$\;\;$\mathbf{y}_{i}, \mathbf{y}_{j} \in$ B \Comment{Mix the targets}
\State loss $\leftarrow$ $-\frac{1}{\vert B \vert} \sum_{i=1}^{\vert B \vert} \mathcal{L} \left( \mathbf{f}_{\Theta_{0}}(\tilde{\mathbf{x}}_{i}), \tilde{\mathbf{y}}_{i} \right)$
\State \textcolor{white}{loss $\leftarrow$} $+\frac{\lambda}{\vert B \vert} \sum_{k=1}^{K} \sum_{(i, j)}^{\vert B \vert} \zeta \cdot \ell^{ALS}_{\tilde{\alpha}(\mathbf{x}_{i})} \left( \mathbf{f}_{\Theta_{k}}(\tilde{\mathbf{x}}_{i}), \mathbf{y}_{i} \right) + (1 - \zeta) \cdot \ell^{ALS}_{\tilde{\alpha}(\mathbf{x}_{j})} \left( \mathbf{f}_{\Theta_{k}}(\tilde{\mathbf{x}}_{i}), \mathbf{y}_{j} \right)$
\State \textcolor{gray}{/* Compute loss by using Eq. \eqref{eq:als} */}
\State update $\Theta_{0}$ and $\left\{ \Theta_{k} \right\}_{k=1}^{K}$ using SGD
\EndFor
\end{algorithmic}
\end{algorithm}

\begin{table*}[ht]
\centering
\caption{Test accuracies (\%) on CIFAR-10 and CIFAR-100 under different noisy types and fractions for noise-robust loss approaches. The results of all baseline methods were taken from original works\,\cite{li2020dividemix, liu2020early}. By following previous works, we report the best and last test accuracy sequentially for each experiment. The best results sharing the noisy fraction and method are highlighted in bold.}
\scriptsize{ 
\begin{tabular}{l|ccccc|ccccc} \toprule
Dataset      & \multicolumn{5}{c}{CIFAR-10}     & \multicolumn{5}{|c}{CIFAR-100}       \\ \midrule
Noisy Type   & \multicolumn{4}{c}{Sym.} & Asym. & \multicolumn{4}{c}{Sym.} & Asym. \\ \midrule
Noise Ratio  & 20\% & 50\% & 80\% & 90\% & 40\% & 20\% & 50\% & 80\% & 90\% & 40\% \\ \midrule \midrule
DivideMix  & 96.1\,/\,95.7 & 94.6\,/\,94.4 & 93.2\,/\,92.9 & 76.0\,/\,75.4 & 93.4\,/\,92.1 & 77.3\,/\,76.9 & 74.6\,/\,74.2 & 60.2\,/\,59.6 & 31.5\,/\,31.0 & 72.1\,/\,- \\
\; \textbf{+ ALASCA} & \textbf{96.3\,/\,95.9} & \textbf{95.3\,/\,94.8} & \textbf{93.5\,/\,93.0} & \textbf{78.7\,/\,78.2} & \textbf{93.5\,/\,92.9} & \textbf{77.4\,/\,77.2} & \textbf{75.0\,/\,74.5} & \textbf{60.9\,/\,59.8} & \textbf{32.3\,/\,32.0} & \textbf{72.2\,/\,71.5} \\ \midrule
ELR+  & 95.8\,/\,94.6 & 94.8\,/\,93.8 & 93.3\,/\,91.1 & 78.7\,/\,75.2 & 93.0\,/\,92.7 & 77.6\,/\,77.5 & 73.6\,/\,72.4 & 60.8\,/\,58.2 & 33.4\,/\,30.8 & 77.5\,/\,76.5 \\
\; \textbf{+ ALASCA} & \textbf{96.0\,/\,95.8} & \textbf{94.9\,/\,94.3} & \textbf{93.4\,/\,92.8} & \textbf{78.8\,/\,77.1} & \textbf{93.8\,/\,93.3} & \textbf{77.9\,/\,77.9} & \textbf{73.8\,/\,73.5} & \textbf{61.0\,/\,60.4} & \textbf{34.2\,/\,34.0} & \textbf{77.8\,/\,77.5} \\
\bottomrule
\end{tabular}
}
\label{tab:semi}
\end{table*}

\subsection{Comparison with Zhang et al. (2019)}\label{sec:byot} 
Our ALASCA can be seemingly similar to BYOT \cite{zhang2019your}, however, as mentioned in Section~\ref{sec:3.2}, there exists a fundamental difference; our method is based on the theoretical motivation that LS encourages LR and uses auxiliary classifiers for explicitly regularizing intermediate layers. Furthermore, BYOT proceeds with both vanilla knowledge distillation (KD) and feature distillation (FD) which has risks to inject corrupted information from network outputs to feature extractors due to label noise. These risks can be led to sub-optimal performances under label noise. Table~\ref{tab:sd} summarizes the results of the comparison of test accuracies between BYOT and ALASCA on CIFAR-10 under various label noises. While both BYOT and ALASCA work well on lower-level of label noise, the performance gap between ALASCA and BYOT becomes larger as the noise level is larger. We also conduct experiments on the case of using only KD or FD of BYOT, and both cases show lower performances than ALASCA. From the results, we verify that our ALASCA is a robust method in presence of noisy labels that is different from the existing BYOT.

\begin{table}[ht]
\caption{Comparison of performances with BYOT under different noise types and fractions (\%) on CIFAR-10. Since they distill corrupted information, we observe that all distillation methods have lower performances than our ALASCA.} \label{tab:sd}
\centering
\scriptsize{
\begin{tabular}{l|cccc} \toprule
 & Symm. 20\% & Symm. 50\% & Asym. 40\% & Inst. 40\% \\ \midrule 
Stndard & 87.0 $\pm$ 0.1 & 78.2 $\pm$ 0.8 & 85.0 $\pm$ 0.1 & 74.1 $\pm$ 2.9 \\ 
BYOT & 91.2 $\pm$ 0.1 & 86.6 $\pm$ 0.1 & 87.7 $\pm$ 0.3 & 76.1 $\pm$ 0.1 \\ 
\;- KD & 91.1 $\pm$ 0.1 & 85.9 $\pm$ 0.2 & 89.4 $\pm$ 0.6 & 80.1 $\pm$ 0.4 \\ 
\;- FD & 91.1 $\pm$ 0.1 & 86.5 $\pm$ 0.2 & 88.0 $\pm$ 1.5 & 76.4 $\pm$ 1.1 \\ \midrule 
ALASCA & \textbf{92.2 $\pm$ 0.2} & \textbf{88.0 $\pm$ 0.3} & \textbf{90.3 $\pm$ 0.3} & \textbf{81.4 $\pm$ 0.3} \\ \bottomrule 
\end{tabular}}
\end{table}

\subsection{Compatibility with Self-supervised Learning.}\label{sec:self}
One of the promising methods under label noise is the self-supervised framework that does not use label information. Since many previous works \cite{cheng2021demystifying, zheltonozhskii2022contrast} show that applying self-supervised learning, which provide robust feature extractor as initialization point, achieves high performances, we evaluate the compatibility of ALASCA with SimCLR \cite{chen2020simple}. We apply a backbone network as ResNet18 with pre-trained weight from \cite{zheltonozhskii2022contrast}. We compare the three approaches: (1) freeze the pre-trained feature extractor and train only linear classifier (linear evaluation); (2) retrain the pre-trained networks without ALASCA; (3) retrain the pre-trained networks with ALASCA. While the previous research shows that self-supervised frameworks are significantly effective in high-level noise ratios, we consider this to be extremely impractical that cannot be found in real-world situations. Instead, we combine ALASCA and self-supervised framework in probable noise rates. For the experiments, we follow the experimental setup of \citet{cheng2021demystifying}. Figure~\ref{fig:contrastive} demonstrates the results on CIFAR-10 under 20\% and 40\% of symmetric, and 40\% of asymmetric and instance-dependent noise rates. We observe that our ALASCA can be compatible with self-supervised frameworks.

\begin{figure*}[ht]
    \hspace*{\fill}
    \begin{subfigure}[b]{0.20\textwidth} 
    \centering
    \includegraphics[width=\linewidth]{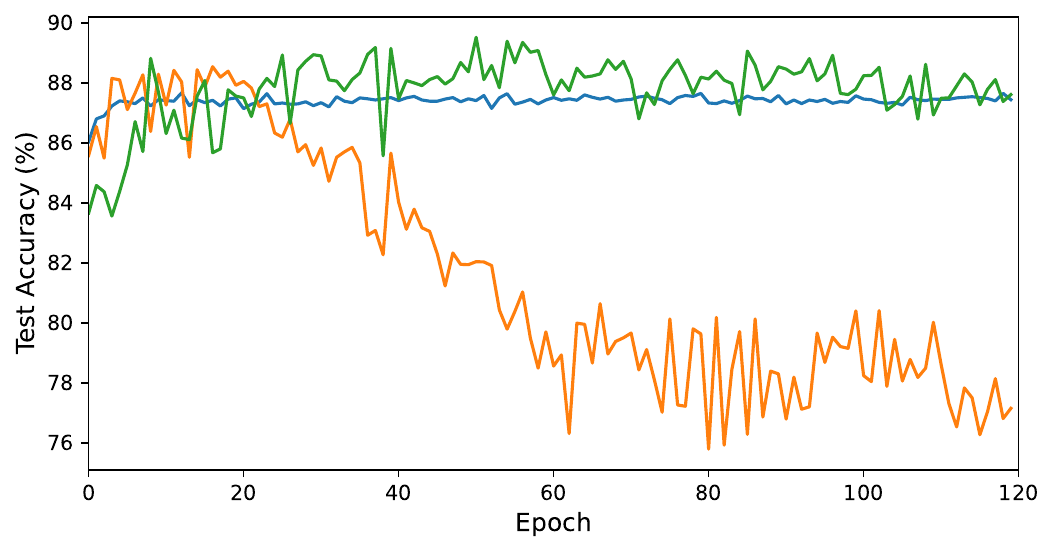}
    \caption{Symm. 20\%}
    \end{subfigure}
    \hfill
    \begin{subfigure}[b]{0.20\textwidth} 
    \centering
    \includegraphics[width=\linewidth]{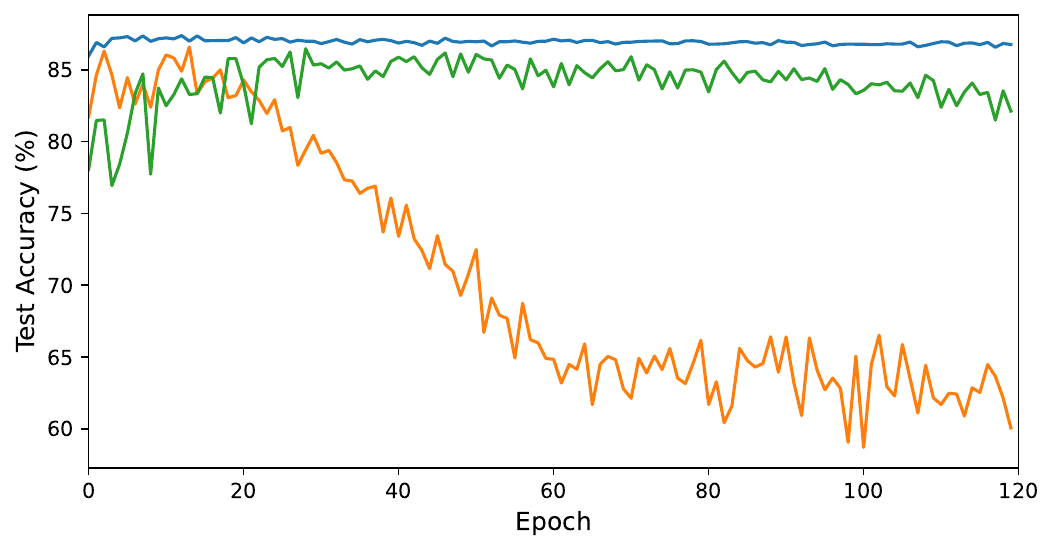}
    \caption{Symm. 40\%}
    \end{subfigure}
    \hfill
    \begin{subfigure}[b]{0.20\textwidth} 
    \centering
    \includegraphics[width=\linewidth]{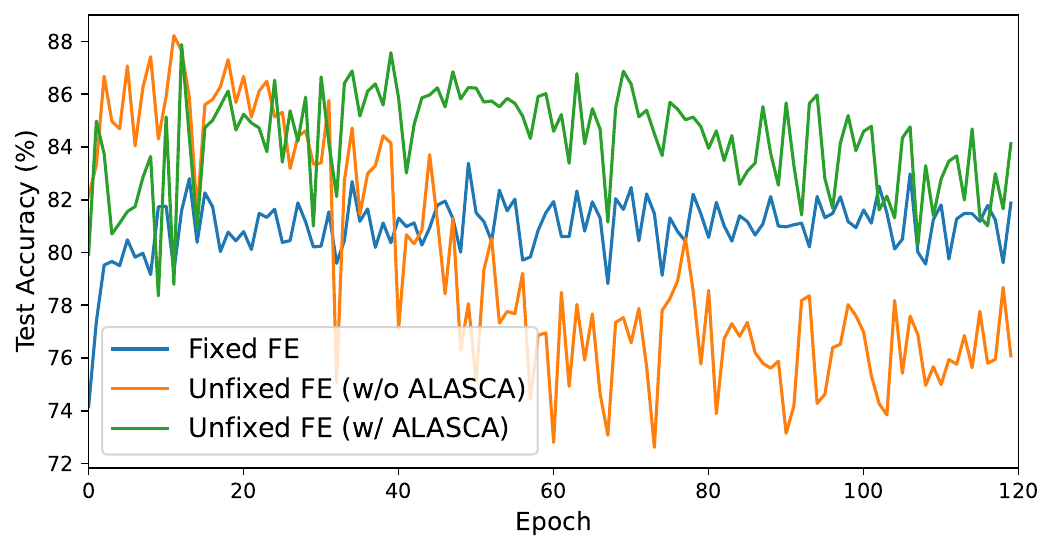}
    \caption{Asym. 40\%}
    \end{subfigure}
    \hfill
    \begin{subfigure}[b]{0.20\textwidth} 
    \centering
    \includegraphics[width=\linewidth]{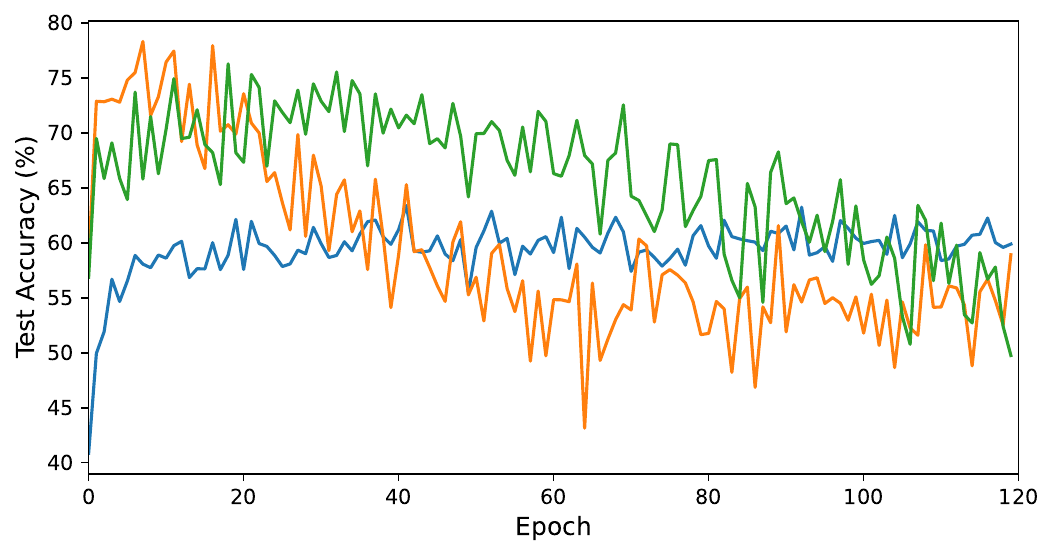}
    \caption{Inst. 40\%}
    \end{subfigure}
    \hspace*{\fill}
    \caption{Comparison for test accuracy across the training epochs. We denote FE as the feature extractor. While the orange lines significantly drop as training proceeds, the green lines retain the performance level in the most of noise distributions. The green lines also show higher performance than the blue lines, except for symmetric noise 40\%, and show competitive performance with the blue line for 40\% of symmetric noise. }\label{fig:contrastive}
\end{figure*}

\subsection{Experimental Results on (mini) WebVision Datset}\label{sec:webvision}
WebVision is a large-scale image dataset collected by crawling Flickr and Google, which resulted in an estimated 20\% of noisy labels \cite{li2017webvision}. In addition to Section~\ref{sec:4.4}, we apply ALASCA to train the model on the mini WebVision dataset and evaluate it on both the WebVision and ImageNet ILSVRC12 validation sets. All methods use an InceptionResNetV2 as backbone architecture and we follow the experimental settings of \citet{liu2020early}. We set the weight decay as $5 \times 10^{-4}$ and batch size as 32. We set the initialized learning rate as 0.02, reduce it by a factor of 10 after 50 epochs (total 100 epochs). We set the noise rate for A-Coteaching as 20\% by following previous studies. Similar to the result in the Clothing1M dataset, our ALASCA performs strongly, despite its simplicity. Furthermore, A-ELR+ and A-Coteaching achieve the highest performance on WebVision and ILSVRC12 validation datasets, respectively. 

\begin{table*}[ht]
\centering
\caption{Comparison with state-of-the-art methods on (mini) WebVision dataset. Numbers denote top-1 (top-5) accuracy (\%) on the WebVision validation set and the ImageNet ILSVRC12 validation set. Results for baselines are copied from original works. A-Coteaching and A-ELR+ denote the methods combining ALASCA with Co-teaching and ELR+.}
\scriptsize{
\begin{tabular}{c|c|c|c|c|c|c|c|c|c} \toprule
 & & Co-teaching & CRUST & FINE & HAR & ELR+ & ALASCA & A-Coteaching & A-ELR+ \\ \midrule
\multirow{2}{*}{WebVision} & top1 & 63.58 & 72.40 & 75.24 & 75.50 & 77.78 & 76.09 & 77.14 & \textbf{77.98} \\
 & top5 & 85.20 & 89.56 & 90.28 & 90.70 & 91.68 & 90.85 & 92.19 & \textbf{92.25} \\ \midrule
 \multirow{2}{*}{ILSVRC12} & top1 & 61.48 & 67.36 & 70.08 & 70.30 & 70.29 & 70.34 & \textbf{71.08} & 70.87 \\
 & top5 & 84.70 & 87.84 & 89.71 & 90.00 & 89.76 & 90.25 & 90.87 & \textbf{91.31} \\ \bottomrule
\end{tabular}}
\label{tab:webvision}
\end{table*}

\subsection{Sensitivity Analysis}\label{sec:sensitivity}
The main hyperparameters of ALASCA are the regularization coefficient $\lambda$, EMA weight $\beta$, and EMA sharpening temperature $\tau$. To understand the effect of each hyperparameter, we conduct sensitivity analysis for each hyperparameter on CIFAR-10 under 50\% of symmetric label noise. Figure~\ref{fig:sensitivity} summarizes the performance of ALASCA for different hyperparameter values. 
We observe that our ALASCA is reasonably robust to its own hyperparameter values.
\begin{figure*}[ht]
    \hspace*{\fill}
    \begin{subfigure}[b]{0.27\textwidth} 
    \centering
    \includegraphics[width=\linewidth]{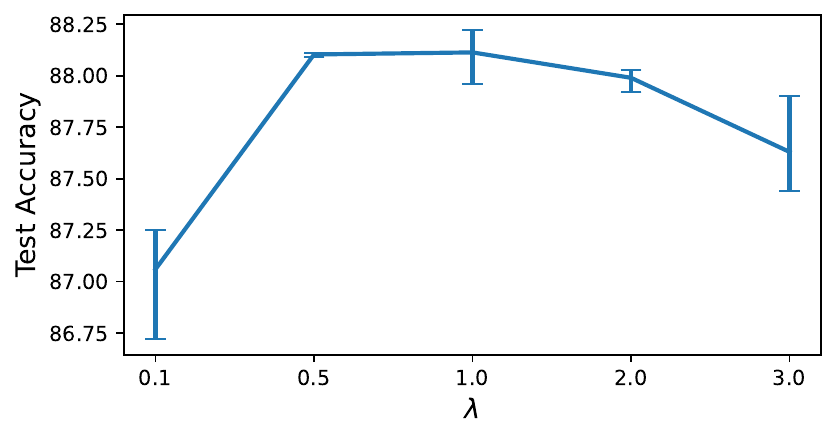}
    \caption{Regularization coefficient ($\lambda$)}
    \end{subfigure}
    \hfill
    \begin{subfigure}[b]{0.27\textwidth} 
    \centering
    \includegraphics[width=\linewidth]{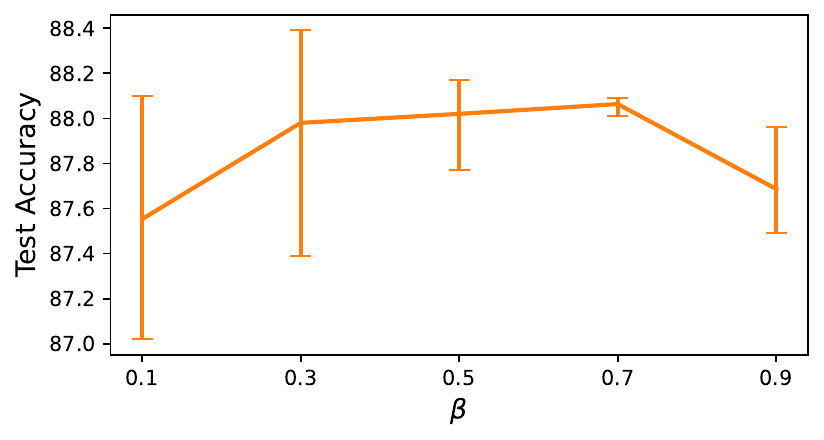}
    \caption{EMA weight ($\beta$)}
    \end{subfigure}
    \hfill
    \begin{subfigure}[b]{0.27\textwidth} 
    \centering
    \includegraphics[width=\linewidth]{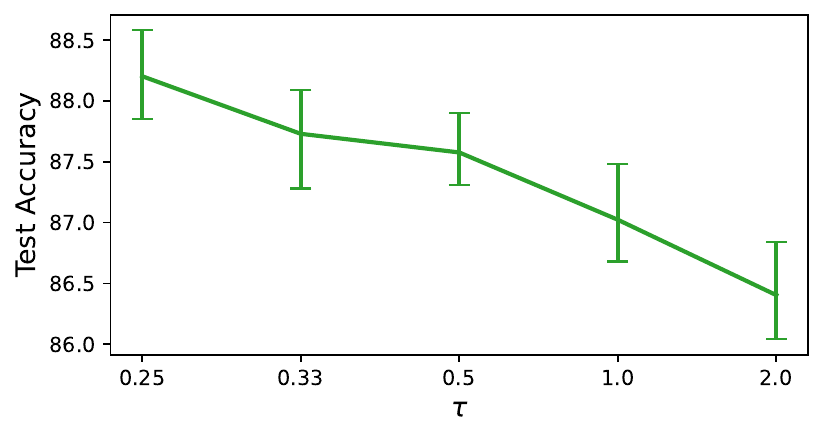}
    \caption{EMA sharpening temperature ($\tau$)}
    \end{subfigure}
    \hspace*{\fill}
    \caption{Test accuracy (\%) for sensitivity analysis on CIFAR-10 under 50\% of noise. The mean accuracy over three runs is reported, along with bars representing the maximum and minimum values. Note that the rest of the hyperparameters are fixed to the values in Section~\ref{sec:4.1}.}\label{fig:sensitivity}
\end{figure*}

\end{document}